\documentclass{article} 
\usepackage{iclr2026/iclr2026_conference,times}

\usepackage{graphicx} 
\usepackage{caption}  
\usepackage{amsmath}
\usepackage{amsthm}
\usepackage{amssymb}
\usepackage{hyperref}
\usepackage{booktabs}
\usepackage{tabularx}
\usepackage{multirow} 
\usepackage{graphicx}
\usepackage{subcaption}
\usepackage{wrapfig}
\usepackage{hyperref}
\usepackage{xcolor}
\usepackage{bm}
\usepackage{bbm}
\usepackage{booktabs}
\usepackage{makecell}
\usepackage{algorithm}
\usepackage{algpseudocode}
\usepackage{subcaption} 
\newtheorem{assumption}{Assumption}

\newtheorem{lemma}{Lemma}
\newtheorem{theorem}{Theorem}
\newtheorem{definition}{Definition}
\newtheorem{proposition}{Proposition}

\usepackage{amsmath,amsfonts,bm}

\def\eqref#1{equation~\ref{#1}}

\def\1{\bm{1}}

\DeclareMathAlphabet{\mathsfit}{\encodingdefault}{\sfdefault}{m}{sl}
\SetMathAlphabet{\mathsfit}{bold}{\encodingdefault}{\sfdefault}{bx}{n}

\algtext*{EndFor}
\algrenewcommand\algorithmicrequire{\textbf{Input:}}
\algrenewcommand\algorithmicensure{\textbf{Output:}}
\usepackage{hyperref}
\usepackage{url}

\usepackage{soul}
\sethlcolor{orange!15}

\title{Revisiting Meta-Learning with Noisy Labels: Reweighting Dynamics and Theoretical Guarantees}

\author{
  \textbf{Yiming Zhang}\textsuperscript{*} \quad
  \textbf{Chester Holtz}\textsuperscript{*} \quad
  \textbf{Gal Mishne} \quad
  \textbf{Alex Cloninger} \\
  University of California, San Diego \\
  \textsuperscript{*}\,Equal contribution
}
\iclrfinalcopy 
\begin{document}

\maketitle

\begin{abstract}
Learning with noisy labels remains challenging because over-parameterized networks memorize corrupted supervision. Meta-learning–based sample reweighting mitigates this by using a small clean subset to guide training, yet its behavior and training dynamics lack theoretical understanding. We provide a rigorous theoretical analysis of meta-reweighting under label noise and show that its training trajectory unfolds in three phases: (i) an alignment phase that amplifies examples consistent with a clean subset and suppresses conflicting ones; (ii) a filtering phase driving noisy example weights toward zero until the clean subset loss plateaus; and (iii) a post-filtering phase in which noise filtration becomes perturbation-sensitive. The mechanism is a similarity-weighted coupling between training and clean subset signals together with clean subset training loss contraction; in the post-filtering where the clean-subset loss is sufficiently small, the coupling term vanishes and meta-reweighting loses discriminatory power. Guided by this analysis, we propose a lightweight surrogate for meta reweighting that integrates mean-centering, row shifting, and label-signed modulation, yielding more stable performance while avoiding expensive bi-level optimization. Across synthetic and real noisy-label benchmarks, our method consistently outperforms strong reweighting/selection baselines.
\end{abstract}

\section{Introduction}

Driven by the rapid growth of accessible data, deep neural networks (DNNs) now achieve state-of-the-art performance across many tasks \citep{he2016deep, krizhevsky2012imagenet}. Yet large datasets are often annotated by humans or scraped from the web, making labels inexpensive but error-prone in practice \citep{yu2018learning, cheng2020learning}. Such label noise—stemming from ambiguity, bias, or systematic mistakes—poses a serious challenge because modern DNNs are expressive enough to memorize arbitrary label assignments \citep{zhang2016understanding}. Consequently, even modest noise levels can mislead training and harm generalization through overfitting to corrupted supervision \citep{zhang2016understanding, arpit2017closer}. Developing and understanding learning methods that remain reliable under noisy labels is therefore a central and urgent problem.

To curb memorization of corrupted labels, a growing literature develops methods for noisy label learning. Among existing strategies, sample reweighting and selection (i.e., ``noise cleansing") are widely used for their strong empirical performance and compatibility with contrastive and unsupervised learning frameworks \citep{li2020dividemix, miao2024learning}. Sample reweighting/selection down-weights or filters suspicious examples using signals such as per-example loss \citep{gui2021towards}, inter-model agreement \citep{wei2020combating}, or cross-validation/large-loss heuristics (e.g., periodically partitioning data and removing high-loss instances) \citep{chen2019understanding}. Within the sample reweighting methods, meta-learning–based reweighting has emerged as a prominent approach: a small trusted validation set guides a bilevel optimization that learns example weights to improve downstream generalization \citep{ren2018learning, shu2019meta, holtz2022learning}. However, scalability remains a key bottleneck since computing hypergradients via unrolling or implicit differentiation incurs substantial memory and compute overhead on modern architectures. Additionally, the theoretical underpinnings that drive the success of meta-learning-based reweighting are still incomplete. It is not yet fully understood when and why meta-reweighting provably suppresses noisy examples, how the quality and bias of the validation set affects generalization, or how inner loop hyperparameters converge. 

Existing convergence results for gradient-based meta-learning guarantee nonasymptotic convergence to first-order stationary points (FOSPs) (e.g., \citep{shu2019meta, ji2022theoretical}). Such optimization guarantees alone do not preclude overfitting to noisy labels. In particular, \citet{zhai2022understanding} shows that, under overparameterization and standard losses, a broad class of sample-reweighting schemes inherits an implicit bias of empirical risk minimization (ERM) and may therefore offer no improvement over ERM. That analysis, however, assumes per-example weights converge to strictly positive limits and ignores the regularization effect introduced by the trusted meta clean set that guides meta-reweighting. These gaps require a sharper theory of meta-reweighting under label noise. In particular, we demonstrate that a careful analysis that models the interaction between weight dynamics and meta-set bias informs the design of more efficient and robust algorithms. A detailed discussion of related work is provided in Appendix~\ref{sec:related}.

In this work, we analyze the training dynamics of meta-reweighting under noisy labels. We show that when the reweighting step size scales inversely proportional to the classifier training step size, the classifier is guaranteed to separate clean from noisy data. More concretely, under standard over-parameterization and stable step sizes, training follows a three-phase trajectory: an early stage that boosts examples aligned with the clean subset signal while suppressing conflicting ones; a filtering stage during which the weights polarize so clean samples dominate, and which terminates when the clean subset error has converged (within a small tolerance); and a post-filtering stage where the signal naturally tapers and polarization becomes susceptible to perturbations. The effect is driven by a signed, similarity-weighted coupling between training and clean subset signals together with residual contraction. Our analysis clarifies when and how label filtering occurs and offers practical guidance for hyperparameters and the design of a particular training methodology. We demonstrate this by introducing a computationally light surrogate for the bilevel update. Specifically, our surrogate retains the signed, similarity-weighted aggregation, eliminates the need for backpropagation through inner-loop updates, and mean-centers the similarity Gram matrix to remove global bias. Empirically, on both synthetic and real-world datasets, the method consistently improves test accuracy over common baselines, demonstrating efficient noisy-label robust learning. We summarize our key contributions as follows:
\begin{itemize}
    \item \textbf{Theoretical analysis.} We rigorously analyze the training dynamics of meta reweighting and characterize three stages. (1) Sample weight alignment stage where noisy weights decrease. (2) Noisy sample weights remain at zero while the loss of the clean subset converges. (3) The weights start to become perturbed and no longer filter out noisy samples. 

    \item \textbf{Algorithmic surrogate.}
    Guided by our theory, we introduce a lightweight surrogate that updates sample weights using penultimate-layer feature similarity in place of the tangent-kernel similarity. To better align with the theoretical suggestions, we incorporate (i) mean-centering and row-shifting to remove global bias, (ii) label-signed modulation—i.e., we set the sign of each pairwise similarity by label agreement. 

    \item \textbf{Empirical evaluation.} We empirically show that our method achieves higher accuracy than sample reweighting/selection baselines across multiple benchmarks and noise settings.
\end{itemize}
\subsection{Notations}
Let $\mathcal{X}\subset\mathbb{R}^d$ and $\mathcal{Y}\subset\mathbb{R}$ denote the input and output spaces, respectively, and assume $\|\mathbf{x}\|_2\le1$ for all $\mathbf{x}\in\mathcal{X}$. 
We observe $\mathcal{D}=\{\mathbf{z}_i=(\mathbf{x}_i,y_i)\}_{i=1}^n$ drawn i.i.d.\ from $P$ over $\mathcal{X}\times\mathcal{Y}$. 
Let $\mathbf{X}=(\mathbf{x}_1,\ldots,\mathbf{x}_n)\in\mathbb{R}^{d\times n}$, the (hypothetical) clean labels $\mathbf{Y}^{\ast}=(y_1^{\ast},\ldots,y_n^{\ast})^\top\in\mathbb{R}^n$, and the observed (possibly noisy) labels $\mathbf{Y}=(y_1,\ldots,y_n)^\top\in\mathbb{R}^n$. In this work, we refer to any sample $\mathbf{x}_i$ with $y_i \neq y_i^{\ast}$ as a \emph{noisy sample}, and to any sample with $y_i = y_i^{\ast}$ as a \emph{clean sample}.
For any mapping $g:\mathcal{X}\to\mathbb{R}^m$, we overload notation as $g(\mathbf{X})=(g(\mathbf{x}_1),\ldots,g(\mathbf{x}_n))\in\mathbb{R}^{m\times n}$ (a column vector when $m=1$).

Let $f_\theta:\mathcal{X}\to\mathbb{R}$ be a neural network with parameters $\theta$ and loss $\ell:\mathbb{R}\times\mathcal{Y}\to\mathbb{R}_+$. Write $f_i(\theta):=f_\theta(\mathbf{x}_i)$ and the per-example loss $l_i(\theta):=\ell\big(f_\theta(\mathbf{x}_i),y_i\big)$. For nonnegative weights $w=(w_1,\ldots,w_n)^\top\in\mathbb{R}_{\ge0}^n$, define the weighted empirical risk $L(\theta,w):=\tfrac{1}{n}\sum_{i=1}^n w_i\,l_i(\theta)$, which reduces to standard ERM when $w_i\equiv1$. In the analysis we take $\mathcal{Y}=\{+1,-1\}$ and the squared loss $\ell(\hat y,y)=\tfrac{1}{2}(\hat y-y)^2$ with $\hat y\in\mathbb{R}$ and $y\in\{+1,-1\}$. Let $u(\theta)=(u_1(\theta),\ldots,u_n(\theta))^\top\in\mathbb{R}^n$ denote residuals with $u_i(\theta):=f_\theta(\mathbf{x}_i)-y_i$, and let $J(\theta):=\nabla_\theta f_\theta(X)\in\mathbb{R}^{p\times n}$ denote the Jacobian. For the clean subset, we use $f_i^v(\theta)$, $u^v(\theta)$, and $J^v(\theta)$ analogously.

 \section{Background and Meta Reweighting Framework}
\label{sec:meta_rw}
There has been a growing line of work on \emph{meta reweighting} for robust training. 
For example, \cite{ren2018learning} propose a framework to learn per-example scalar weights by differentiating a validation objective through a one-step proxy of the training update, thereby steering optimization toward samples that improve held-out performance (e.g., under label noise or class imbalance). \cite{shu2019meta} parametrize and learn a smooth weighting function by optimizing a meta-objective on a clean validation set, mapping example statistics (such as the loss suffered on each training example) to adaptive weights.
Motivated by the implicit bilevel mechanism of these methods, we adopt a formulation that treats $w \in [0,1]^n$ as an optimizable parameter. At each outer iteration $t$, we (i) update $w$ by performing one inner gradient steps on the validation objective evaluated at a one-step proxy $\widehat\theta_{t+1}(w)$ of the parameter update, and (ii) update $\theta$ once using the resulting weights $w_{t+1}$. This abstraction retains the essential mechanism of meta reweighting without method-specific design choices and serves as the basis for our analysis in the next sections.

\noindent \textbf{Pseudo update step.}
At outer iteration $t$, given model parameters $\theta_t$ and example weights $w_t=(w_t^{\,1},\dots,w_t^{\,n})$, we consider the following one-step proxy that maps weights to a tentative parameter update:
\[
\widehat{\theta}_{t+1}(w)
\;:=\;
\theta_t \;-\; \eta\, \nabla_{\theta}\!\Big(\sum_{i=1}^n  w^{\,i}\, l_{i}(\theta_t)\Big),
\]
where $l_i(\theta)$ denotes the training loss of example $i$ and $\eta>0$ is the classifier learning rate. 
Using the clean subset losses $\{l^v_j(\cdot)\}_{j=1}^m$, define the meta objective as a function of the weights
\[
\mathcal L_{\mathrm{val}}(w)
\;:=\;
\frac{1}{m}\sum_{j=1}^m l^v_j\!\big(\widehat{\theta}_{t+1}(w)\big).
\]
\noindent \textbf{Reweighting parameter update.}
Given $w_t$,  take one gradient step on $\mathcal L_{\mathrm{val}}$ with step size $\alpha>0$:
\[
\widetilde w_{t+1}
\;=\;
w_t \;-\; \alpha\, \nabla_w \mathcal L_{\mathrm{val}}(w_t).
\]
Equivalently, unrolling the pseudo update step by the chain rule yields the following scheme:
\begin{align}
w^i_{t+1} = w^i_{t} - \alpha \frac{\partial}{\partial w^i}( \sum_{j=1}^m l^v_j(\theta_t-\eta \nabla_\theta \sum_{k=1}^n  w_t^k l_{k}(\theta_t))) ,
\qquad i=1,\dots,n.
 \label{alg:meta}
\end{align} 
In this work, we choose the reweighting stepsize as $\alpha=\beta/\eta$ with a fixed constant $\beta>0$, so that the effective product $\alpha\eta$ remains constant; indeed, the $w$-update in Equation~(\ref{alg:meta}) depends on $\eta$ only through this product. We require the weights to lie in $[0,1]$. Therefore, after each step we apply elementwise clipping:
\[
w_{t+1}^{\,i}
\;=\;
\min\!\big\{\,1,\ \max\{\,0,\ \widetilde w_{t+1}^{\,i}\,\}\big\},
\qquad i=1,\dots,n.
\]


\noindent \textbf{Classifier parameter update.}
Using the updated weights $w_{t+1}$, perform one classifier update of the model parameters:
\begin{equation}\label{eq:outer-update}
\theta_{t+1}
= \theta_t - \eta\, \nabla_{\theta}\!\Big(\sum_{i=1}^n  w_{t+1}^{\,i}\, l_{i}(\theta_t)\Big).
\end{equation}


\section{Theoretical Analysis}
\label{sec:theoretical}
In this section, we will study the learning dynamics of the meta reweighting algorithm in Section \ref{sec:meta_rw} under certain assumptions. Specifically, we will prove that at the early phase of training, the weights of noisy samples decrease to zero while weights of clean samples increase to one. Then, the weights will stay constant until the training loss of clean subset converges to $O(\eta+\widetilde d^{-1/4})$. Finally, since the training loss of the clean subset is small, the weight gradient signal is overwhelmed by approximation errors. As a result, the weight update starts to be perturbed and noisy sample weights could be potentially positive.
 
 We consider training dynamics of sufficiently wide fully-connected neural networks. Specifically, we define the output of the $l+1$-th hidden layer of a $L$-layer neural network to be
$$
\boldsymbol{h}^{l+1}=\frac{A^l}{\sqrt{d_l}} \boldsymbol{x}^l+\boldsymbol{b}^l \quad \text { and } \quad \boldsymbol{x}^{l+1}=\sigma\left(\boldsymbol{h}^{l+1}\right) \quad(l=0, \cdots, L)
$$
where $\sigma$ is a non-linear activation function, $A^l \in \mathbb{R}^{d_{l+1} \times d_l}$ and $\bm b^l\in\mathbb{R}^{d_{l+1}}$. Here $d_0=d$. The parameter vector $\theta$ consists of $A^0, \cdots, A^L$ and $\bm b^0, \cdots, \bm b^L$ ( $\theta$ is the concatenation of all flattened weights and biases). The final output is $f(\boldsymbol{x})=\boldsymbol{h}^{L+1}$. And let the neural network be initialized as
$$
\left\{\begin{array} { l } 
{ A _ { i , j } ^ { l  } \sim \mathcal { N } ( 0 , 1 ) } \\
{ b _ { j } ^ { l } \sim \mathcal { N } ( 0 , 1 ) }
\end{array} \quad ( l = 0 , \cdots , L - 1 ) \quad \text { and } \quad \left\{\begin{array}{l}
A_{i, j}^{L}=0 \\
b_j^{L}=0
\end{array}\right.\right.
$$
We also rely on the following assumption for the activation function:
\begin{assumption}
\label{assum:activation}
$\sigma$ is differentiable everywhere. Both $\sigma$ and its first-order derivative $\dot{\sigma} $ are Lipschitz.
\end{assumption}

\textbf{Wide Neural Networks:}
The \textit{neural tangent kernel} (NTK) of model $f_\theta$ is defined as $K^{(0)}(\bm{x}, \bm x^{\prime}) := \nabla_\theta f^{(0)}(\bm{x})^{\top} \nabla_\theta f^{(0)}\left(\bm{x}^{\prime}\right)$. We denote the Gram matrix by $\bm K:= K^{0}(\bm X, \bm{X} ) \in \mathbb{R}^{n \times n}$. When the width of $f_\theta$ increases to infinity, the NTK matrix converges to a deterministic kernel matrix:
\begin{lemma}[\cite{zhai2022understanding}]
    If $\sigma$ is Lipschitz and $d_l \rightarrow \infty$ for $l=1, \cdots, L$ sequentially, then $K^{(0)}\left(\boldsymbol{x}, \boldsymbol{x}^{\prime}\right)$ converges in probability to a deterministic limiting kernel $ K\left(\boldsymbol{x}, \boldsymbol{x}^{\prime}\right)$ and the limiting kernel Gram matrix $\bm K: = K(\bm X, \bm{X}) \in \mathbb{R}^{n \times n}$ is a positive definite symmetric matrix almost surely.
\end{lemma}
We denote the largest and smallest eigenvalues of $\bm K$ by $\lambda_{\max}$ and $\lambda_{\min}$. Throughout this section, the network and data satisfy the following assumptions:
\begin{assumption} \label{assum:net_data}  (i) $d_1=\cdots=d_L=\widetilde d$; (ii) $\{\nabla_\theta f^{(0)}(\boldsymbol{x}_i)\}_{i=1}^n$ are linearly independent; and (iii) $\lambda_{\min}>0$. \end{assumption}

For simplicity, we consider the data distribution support with the following property:
\begin{assumption}
\label{assum:kernel_signs}
If $\bm x, \bm x^{\prime}$ belongs to the same class, $K^{(0)}(\bm x, \bm x^{\prime}) > \gamma >0$. Otherwise $K^{(0)}(\bm x, \bm x^{\prime}) < -\gamma< 0$.
\end{assumption}
We will explain how this assumption can be relaxed in Section \ref{sec:kernel value distribution} and provide empirical evidence to demonstrate that this assumption and its relaxation are satisfied in practice.

\subsection{Learning Dynamics}
We now state our main result on the learning dynamics of the meta reweighting algorithm.
Under the assumptions above and with an appropriately chosen learning rate (namely, $\eta\le\widetilde\eta$ and reweighting step size scales as $\alpha\eta = \beta$), the trajectory exhibits, with high probability, a three-phase behavior: an \emph{early phase} in which clean examples are upweighted and noisy examples are downweighted; a \emph{filtering phase} in which the weights polarize to clean = 1 and noisy = 0
and the validation residual converges toward zero; and a \emph{post-filtering phase} in which the residual-driven signal weakens and noisy sample weights are no longer guaranteed to be zero. The formal statement follows.
\begin{theorem}
\label{thm: main}
Let $f_{\theta_t}$ denote a fully connected neural network satisfying Assumptions~\ref{assum:activation} and \ref{assum:net_data}. Suppose the data distribution satisfies Assumption \ref{assum:kernel_signs}. 
Then there exists $\widetilde\eta>0$ such that, if $\eta\le\widetilde\eta$ and $\alpha = \Omega(\eta^{-1})$, the following holds:
for any $\delta>0$ there exists $D>0$ so that, whenever $\widetilde d\ge D$, with probability at least $1-\delta$ over random initialization,
there exist $0<T_1<T_2$ for some $T_1=1+(m\alpha \eta \gamma)^{-1}$ such that:
\begin{enumerate}
\item (\emph{Early phase}) For $0<t\le T_1$, $w_t^i\in[1/2,1]$ for all clean samples and $w_t^i\in[0,1/2]$ for all noisy samples. Moreover, at $t=T_1$ the weights polarize to the extremes: $w_{T_1}^i=1$ for all clean samples and $w_{T_1}^i=0$ for all noisy samples.
\item (\emph{Filtering phase}) For $T_1<t\le T_2$, $w_t^i$ remains zero on noisy samples; moreover,
\[
\|f_{\theta_{T_2}}(x^v)-y^v\|_\infty = O(\eta+ \widetilde d^{-1/4}).
\]
\item (\emph{Post-filtering phase}) For all $t\ge T_2$, the weights are no longer guaranteed to filter out the noisy samples.
\end{enumerate}
\end{theorem}

The proof is based on the first order approximation of update Equation~(\ref{alg:meta}). We can directly compute $\nabla_w\sum_{j=1}^m l^v_j(\theta_t-\eta \nabla \sum_{i=1}^n  w_t^i l_{i}(\theta_t)) $ via the chain rule:
\begin{align*}
    \nabla_w\sum_{j=1}^m l^v_j(\theta_t-\eta \nabla \sum_{i=1}^n  w_t^i l_{i}(\theta_t)) 
    &= -\eta U(\theta_t) J(\theta_t)^T J^v(\theta_t-\eta \nabla_{\theta}  w_t^T l(\theta_t)) u^v(\theta_t-\eta \nabla_{\theta}  w_t^T l(\theta_t)) \\
    &= -\eta U(\theta_t) J(\theta_t)^T J^{v}(\theta_t) u^v(\theta_t) + e_1\\
    &= -\eta U(\theta_t) K(X, X_{\textrm{clean}}) u^v(\theta_t) + e_1+e_2,
\end{align*}
where \(e_{1}=O(\eta^{2})\) is the truncation error resulting from first-order Taylor expansions of 
\(J^{v}\!\bigl(\theta_{t}-\eta\nabla_{\theta}(w_{t}^{\top}\ell(\theta_{t}))\bigr)\) and 
\(u^{v}\!\bigl(\theta_{t}-\eta\nabla_{\theta}(w_{t}^{\top}\ell(\theta_{t}))\bigr)\) around \(\theta_{t}\), 
and \(e_{2}\) denotes the discrepancy introduced by approximating \(J(\theta_{t})^{\top}J^{v}(\theta_{t})\) with the NTK kernel \(K(X,X_{\mathrm{clean}})\). 

At the beginning of training, the residuals $U(\theta_t)$ (and $u^v(\theta_t)$) are dominated by $\bm y$ (and $\bm y^v)$ because of the zero-prediction initialization. Thus, the update direction of $\bm w$ is controlled by $\bm y K(X, X_{\textrm{clean}})\bm y^v$. Under Assumption \ref{assum:kernel_signs}, weights of noisy (clean) samples will decrease (increase) according to the weight update formula. This dynamic remains true until the residuals are no longer controlled by $\bm y$ (and $\bm y^v$) or the weight update signal $ U(\theta_t) K(X, X_{\textrm{clean}}) u^v(\theta_t) $ is smaller than the error term $e_1+e_2$. Thus at the post-filtering phase, the weights start to be perturbed. A binary-MNIST experiment verifies this mechanism: as shown in Fig.~\ref{fig:weight_dynamics}, the weights traverse the three phases described above; in the late stage, noisy-sample weights drift slightly from zero and evolve only slowly, which can lead to overfitting to label noise.
\subsection{Data Distributions that are Well-Separated in Kernel Space}
\label{sec:kernel value distribution}
It is demonstrated in the proof of Theorem~\ref{thm: main} that Assumption~\ref{assum:kernel_signs} is sufficient for our theoretical guarantees but not necessary. It requires the kernel to induce well-separated, label-consistent clusters: same-class pairs have positive tangent similarity, cross-class pairs have negative tangent similarity. We next investigate how our results extend under weaker conditions, relaxing Assumption~\ref{assum:kernel_signs}. 

\noindent \textbf{Shifted Kernel Values  } 
We can replace Assumption \ref{assum:kernel_signs} with the following:

\begin{assumption}
    \label{assum:shift_sign}
    There exists a constant $\mu$ such that if $\bm x, \bm x^{\prime}$ belongs to the same class, $K^{(0)}(\bm x, \bm x^{\prime}) - \mu > \gamma >0$. Otherwise $K^{(0)}(\bm x, \bm x^{\prime}) - \mu < -\gamma< 0$. Moreover, we assume $\frac{1}{m}|\sum_i u^v_i(\theta_t)| \lesssim \widetilde d^{-1/4} $.
\end{assumption}

Under this relaxed assumption, the proof of the main result follows from the same idea. We empirically verify Assumption \ref{assum:shift_sign} by plotting the histogram of NTK values (Fig.~\ref{fig:NTK_values}) and the mean of residuals over the clean meta subset (Fig.~\ref{fig:sum_res}). 

\noindent \textbf{Mean-Centered Kernel Values  }
Let $K_{\bm\mu}$ denote the $\bm\mu$-centered kernel for some vector $\bm\mu$, then
\begin{align}
u_i(\theta_t)\, K_{\bm\mu}(x_i, X_{\mathrm{clean}})\, u^v(\theta_t)
&= u_i(\theta_t)\,\big( \nabla_{\theta_t}f_{\theta_t}(x_i) - {\bm\mu} \big)^\top
      \big( \nabla_{\theta_t}f_{\theta_t}(X_{\mathrm{clean}}) - {\bm\mu} \mathbf{1}_m^\top \big)\, u^v(\theta_t) \notag\\
&= u_i(\theta_t)\Big[
      K(x_i, X_{\mathrm{clean}})\, u^v(\theta_t)
    - \nabla_{\theta_t}f_{\theta_t}(x_i)^\top {\bm\mu} \;\big(\mathbf{1}_m^\top u^v(\theta_t)\big) \notag\\
&\qquad
    - {\bm\mu}^\top \nabla_{\theta_t}f_{\theta_t}(X_{\mathrm{clean}})\, u^v(\theta_t)
    + {\bm\mu}^\top {\bm\mu} \;\big(\mathbf{1}_m^\top u^v(\theta_t)\big)
    \Big] \label{eq:mean_centered_update}
\end{align}
Since the clean subset is balanced, we further obtain from \eqref{eq:mean_centered_update} that 
\begin{align*}
    u_i(\theta_t)\, K_\mu(x_i, X_{\mathrm{clean}})\, u^v(\theta_t)
&= u_i(\theta_t)\Big(
      K(x_i, X_{\mathrm{clean}})\, u^v(\theta_t) +c_0
    \Big)
\end{align*}
where $c_0 = - {\bm\mu}^\top \nabla_{\theta_t}f_{\theta_t}(X_{\mathrm{clean}})\, u^v(\theta_t),$ which is independent of the individual sample \(x_i\). Since $u_i = -y_i$ at the beginning of training, the weight update direction is thus $y_i K(x_i, X_{\mathrm{clean}}) \bm y^v - y_ic_0$. 

\begin{proposition}
\label{prop:mean_centered}
Let $X_{\mathrm{clean}}=\{x^{\mathrm{clean}}_j\}_{j=1}^m$ be i.i.d.\ samples from a distribution $\mathcal D$.
Let $K$ be an NTK kernel and assume it is bounded: $|K(x,y)|\le K_{\max}$ for all $x,y$.
Assume non-degeneracy: for any fixed $x_i$,
\[
\mathrm{Var}_{X'\sim\mathcal D}\!\big[K(x_i,X')\big]\;\ge\;\sigma^2\;>\;0.
\]
Let the clean labels be balanced and set $\bm u^v=-\bm y^v\in\{\pm1\}^m$ so that $\mathbf 1_m^\top u^v=0$ and $\|\bm u^v\|_2=\sqrt m$.
Write $K_m:=K(X_{\mathrm{clean}},X_{\mathrm{clean}})\in\mathbb R^{m\times m}$ and suppose there exists (possibly random, i.e.\ data-dependent) $c_m\in\mathbb R$ such that
\begin{equation}\label{eq:RS-eps}
\|K_m\mathbf 1_m - c_m\mathbf 1_m\|_2 \;=\; O_p\!\big(m^{1-\varepsilon}\big)
\qquad\text{for some }\ \varepsilon>0.
\end{equation}
Define
\[
c_0 \;:=\; -\,\frac{1}{m}\,\mathbf 1_m^\top K_m\,\bm u^v,\qquad
S_i \;:=\; K(x_i,X_{\mathrm{clean}})\,\bm u^v .
\]
Then, as $m\to\infty$,
\[
|c_0| \;=\; O_p\!\big(m^{\tfrac12-\varepsilon}\big),\qquad
|S_i| \;=\; \Omega_p(\sqrt m),\qquad
\frac{|c_0|}{|S_i|}\;=\; O_p\!\big(m^{-\varepsilon/2}\big)\ \xrightarrow{p}\ 0.
\]
\label{prop:kean_centered}
\end{proposition}
Proposition \ref{prop:kean_centered} implies that, for the first step analysis, Assumption \ref{assum:kernel_signs} can be replaced with the same assumption on the corresponding mean-centered kernel and the row-sum regularity condition \ref{eq:RS-eps}. 

To validate the kernel value distribution assumption for the mean-centered kernel, we conduct experiments on a binary MNIST task and plot the value histogram during training in Fig.~\ref{fig:centered_NTK_values}. We also plot the histogram of the NTK values in Fig.~\ref{fig:NTK_values}.

\section{Revisiting Meta Learning-Based Algorithms}

\begin{algorithm}[t]
\caption{Feature-Based Reweighting (FBR): a lightweight surrogate for meta reweighting}
\label{alg:FBR}
\begin{algorithmic}[1]
\Require Training set $\{(x_i,y_i)\}_{i=1}^n$, clean subset $\{(x_j^v,y_j^v)\}_{j=1}^m$, feature map $\phi:\mathcal X\!\to\!\mathbb R^d$, stepsize $\alpha>0$, label-signed coefficients $(\lambda_+,\lambda_-)\!\ge\!0$, iterations $T$.
\Ensure Weights $\bm w_T$, network $f_{\theta_T}$.
\State Initialize $w_i^{\,0}=\frac{1}{2}$ for all $i$.
\For{epoch $t=0,\cdots, T-1$}
\State $\mu \gets \tfrac{1}{m}\sum_{j=1}^m \phi(x_j^v)$ \Comment{clean-set mean}
\State $\widetilde\phi(x_j^v)\gets \phi(x_j^v)-\mu$ for $j=1,\dots,m$; \quad $\widetilde\Phi_v\gets [\,\widetilde\phi(x_j^v)\,]_{j=1}^m$
\For{mini-batch $B\subseteq\{1,\dots,n\}$}
  \State $\widetilde\phi(x_i)\gets \phi(x_i)-\mu$ for all $i\in B$; \quad $\widetilde\Phi_B\gets[\,\widetilde\phi(x_i)\,]_{i\in B}$
  \State $K_B \gets \widetilde\Phi_B\,\widetilde\Phi_v^{\!\top}$ \Comment{$|B|\!\times\! m$ mean-centered Gram}
  \For{each $i\in B$}
     \State $s_{i,c}\gets \frac{1}{|\{j:\,y_j^v=c\}|}\sum_{j:\,y_j^v=c} (K_B)_{i,j}$ for $c=1,\dots,C$
     \State $c_i^{(1)}\gets \arg\max_{c} s_{i,c}$;\quad $c_i^{(2)}\gets \arg\max_{c\neq c_i^{(1)}} s_{i,c}$ \Comment{top-1 / top-2 means}
     \State $\widetilde K_{i,j}\gets (K_B)_{i,j}-s_{i,c_i^{(2)}}$ \textbf{for} $j=1,\dots,m$ \Comment{row shifting}
     \State $\bar K_{i,j}\gets\left(\lambda_+\mathbf 1\{y_i=y_j^v\}-\lambda_-\mathbf 1\{y_i\neq y_j^v\}\right)\widetilde K_{i,j}$ \textbf{for} $j=1,\dots,m$ 
     \State $\widetilde d_i(\theta_t)\gets \sum_{j=1}^m \bar K_{i,j}$ \Comment{row-sum direction}
     \State $w_i^{\,t+1}\gets \mathrm{clip}_{[0,1]}\!\big(w_i^{\,t}-\alpha\,\widetilde d_i(\theta_t)\big)$
  \EndFor
    \State $\mathcal{L}_B \gets \frac{1}{|B|}\sum_{i\in B} w_i^{\,t}\cdot \ell\!\big(f_{\theta_t}(x_i),\,y_i\big)$
    \State $\theta_{t+1} \gets \theta_t - \eta \nabla_\theta \mathcal{L}_B$
\EndFor
\EndFor
\end{algorithmic}
\end{algorithm}

Our theoretical analysis indicates that the effectiveness of meta reweighting is governed by the \emph{weight–update direction}
\[
d_i(\theta_t)\;=\;u_i(\theta_t)\,\big[K(x_i,X_{\mathrm{clean}})\,u^v(\theta_t)\big],\qquad i=1,\dots,n.
\]
This expression aggregates, for each training example $x_i$, the similarity between $x_i$ and the clean subset $X_{\mathrm{clean}}$ weighted by the meta signals $u_i(\theta_t)$ and $u^v(\theta_t)$.
As established in Theorem~\ref{thm: main}, if $u^v(\theta_t)$ decays toward zero, the useful update signal can be dominated by approximation error; moreover, even under relaxed distributional assumptions (cf.\ Assumption~\ref{assum:shift_sign} and Proposition~\ref{prop:mean_centered}), the direction is susceptible to perturbations when the corresponding constants are large or $|X_{\mathrm{clean}}|$ is small.
In addition, the bilevel structure entails nontrivial computational overhead.
Motivated by these observations, we propose a streamlined procedure for the multiclass setting that preserves the analytical structure of $u(\theta_t)K(\cdot,X_{\mathrm{clean}})u^v(\theta_t)$ while avoiding heavy nested differentiation. 

\begin{figure}[t]
  \centering
  \begin{subfigure}[t]{0.32\textwidth}
    \includegraphics[width=\linewidth]{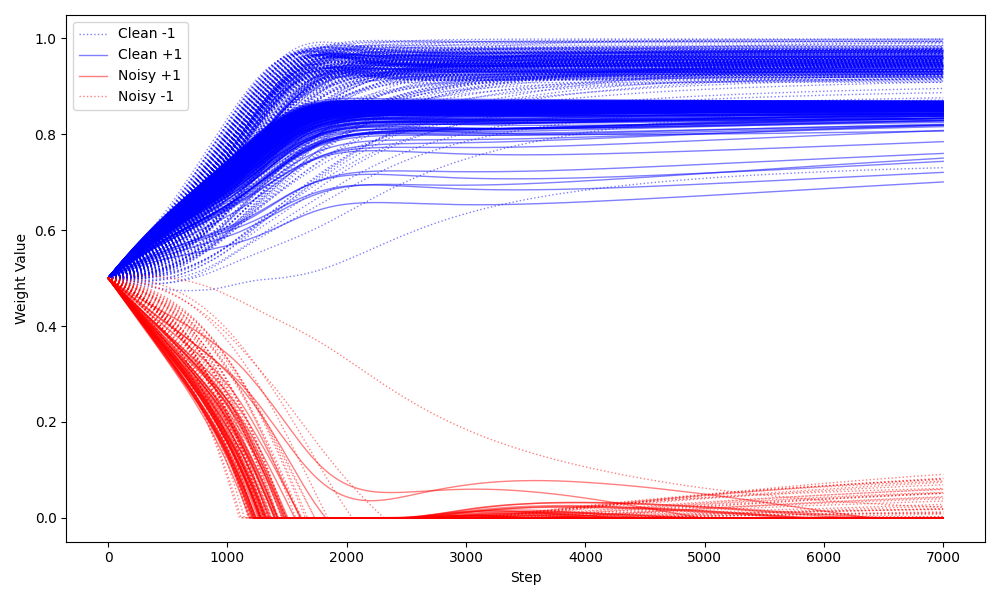}
    \caption{Weight Dynamics}
    \label{fig:weight_dynamics}
  \end{subfigure}\hfill
  \begin{subfigure}[t]{0.32\textwidth}
    \includegraphics[width=\linewidth]{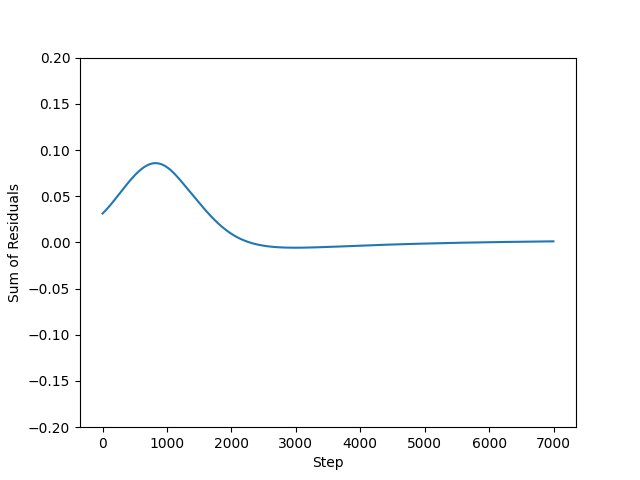}
    \caption{Mean Residual Dynamic}
    \label{fig:sum_res}
  \end{subfigure}\hfill
  \begin{subfigure}[t]{0.32\textwidth}
    \includegraphics[width=\linewidth]{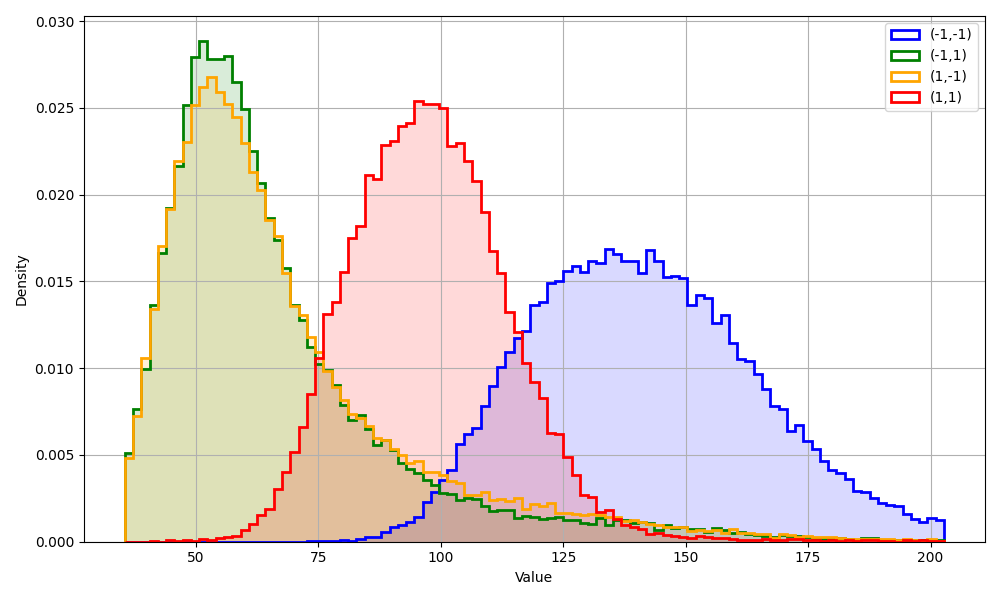}
    \caption{NTK Values}
    \label{fig:NTK_values}
  \end{subfigure}

  \medskip

  \begin{subfigure}[t]{0.32\textwidth}
    \includegraphics[width=\linewidth]{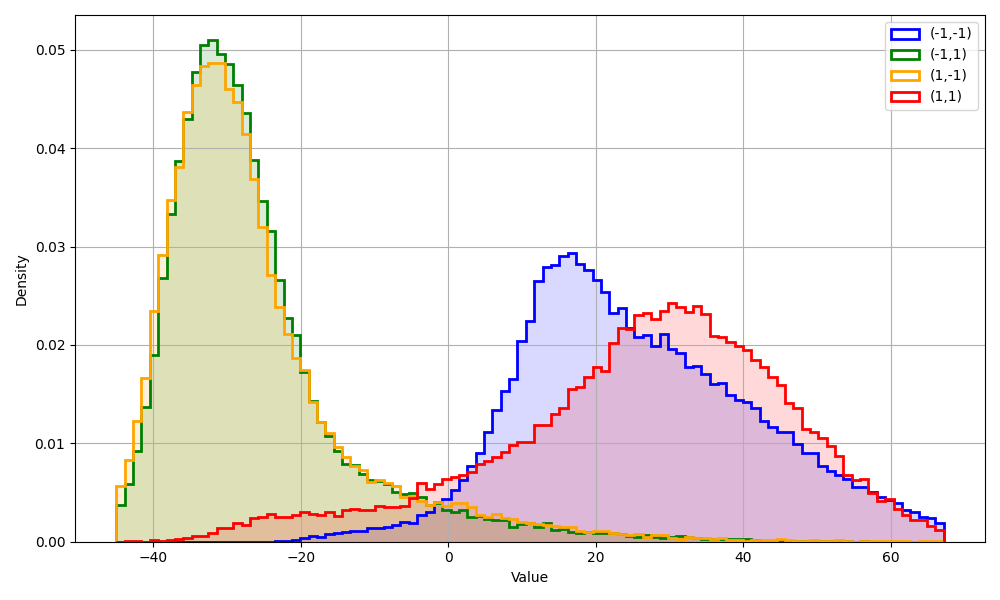}
    \caption{Mean-Centered NTK Values}
    \label{fig:centered_NTK_values}
  \end{subfigure}\hfill
  \begin{subfigure}[t]{0.32\textwidth}
    \includegraphics[width=\linewidth]{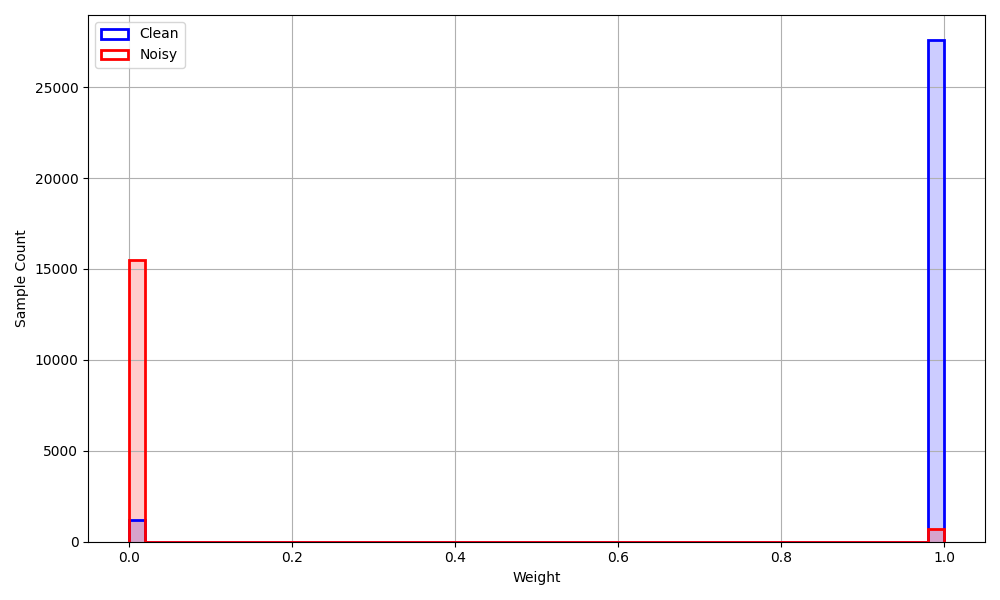}
    \caption{Weight Distribution}
    \label{fig:weight_final}
  \end{subfigure}\hfill
  \begin{subfigure}[t]{0.32\textwidth}
    \includegraphics[width=\linewidth]{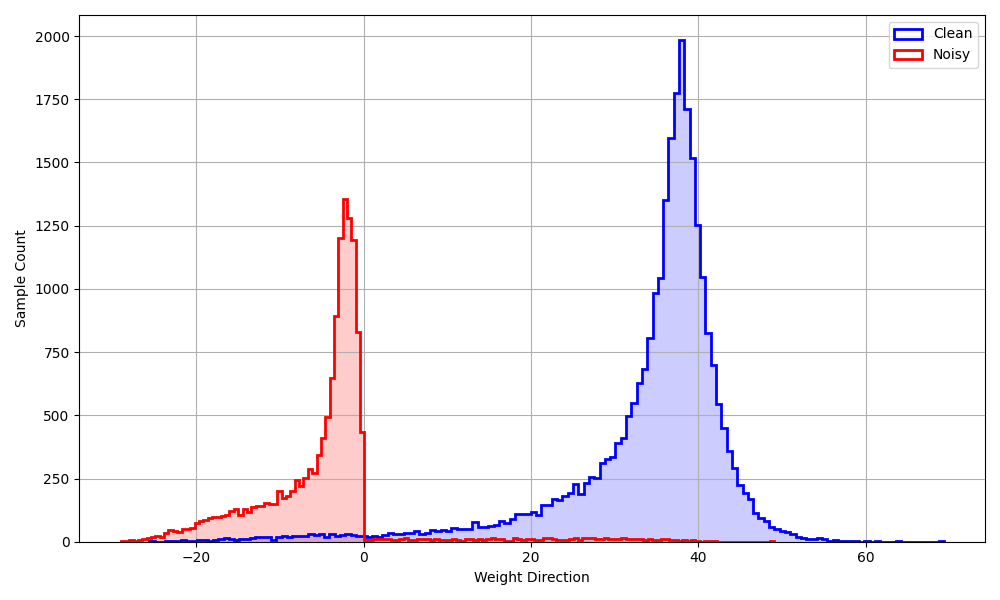}
    \caption{Weight Directions}
    \label{fig:weight_direction_final}
  \end{subfigure}

  \caption{(a) Trajectories of per-sample weights over epochs. (b) Mean residual dynamics on the clean subset. (c) NTK values between training and clean samples. (d) Mean-centered NTK values between training and clean samples. (e) Final weight distribution at convergence, separated by clean vs. noisy samples. (f) Final weight directions.}
  \label{fig:plot}
\end{figure}

\noindent\textbf{(1) Mean-centered features.}
Let $\phi:\mathcal X\to\mathbb R^d$ be a  representation mapping, e.g., the penultimate-layer feature or the neural tangent feature.
Define the clean-set mean $\mu:=\frac{1}{m}\sum_{j=1}^m \phi(x_j^v)$ and the centered features
\(
\widetilde\phi(x):=\phi(x)-\mu.
\)
We then form a Gram matrix between the training and clean sets by
\[
K(x_i,x_j^v)\;:=\;\langle \widetilde\phi(x_i),\,\widetilde\phi(x_j^v)\rangle,
\quad\text{equivalently}\quad
K=\widetilde\Phi\,\widetilde\Phi_v^{\!\top},
\]
where $\widetilde\Phi\in\mathbb R^{n\times d}$ and $\widetilde\Phi_v\in\mathbb R^{m\times d}$ stack the centered features. 
Mean-centering removes the global bias discussed in Section \ref{sec:kernel value distribution}. As indicated in our theoretical analysis, the Gram matrix entries are expected to be well-separated by class: within-class similarities exceed cross-class ones. 

\noindent\textbf{(2) Multiclass row shifting.} Let $C$ be the number of classes. For each point $i$ and class $c$, compute class-wise similarity aggregates
\[
s_{i,c}\;:=\;\frac{1}{| \{ j:\,y_j^v=c \} |}\sum_{j:\,y_j^v=c} K(x_i,x_j^v),\qquad c\in\{1,\dots,C\},
\]
and let $c_i^\star$ be the class with the second-largest mean similarity $s_{i,c}$.
Define a row shift that preserves only the dominant class as positive:
\[
\widetilde K_{i,j}\;:=\;K_{i,j}\;-\;s_{i,c_i^\star}.
\]
This step produces a class-discriminative margin at the row level without modifying the dominant class ordering.

\noindent\textbf{(3) Label-aware scaling.}
To further align the update with the putative label of $x_i$, apply a label-aware multiplicative mask with $\lambda_+,\lambda_-\ge 0$:
\[
\bar K_{i,j}\;:=\;\big(\lambda_+\mathbbm 1\{y_i=y_j^v\}-\lambda_-\mathbbm 1\{y_i\neq y_j^v\}\big)\,\widetilde K_{i,j},
\]
where $\mathbbm 1\{ \cdot\}$ is the indicator function. This operation can be viewed as a simple surrogate for $u_i(\theta_t)u^v(\theta_t)$-weighted alignment when reliable validation signals are weak or noisy.

\noindent\textbf{(4) Row-sum weight update and constraints.}
Define the simplified direction
\[
\widetilde d_i(\theta_t)\;:=\;\,\sum_{j=1}^m \bar K_{i,j},
\]
and update
\[
w^{\,i}_{t+1}\;=\;\Pi_{[0,1]}\!\big(w^{\,i}_t-\alpha\,\widetilde d_i(\theta_t)\big),
\qquad i=1,\dots,n,
\]
where $\Pi_{[0,1]}$ denotes elementwise clipping.
This preserves the signed aggregation structure of the ideal meta direction while replacing the NTK with a centered computationally friendly surrogate.

\noindent\textbf{Remarks on computation and stability.}
(i) The dominant cost is the matrix product $K=\widetilde\Phi\,\widetilde\Phi_v^{\!\top}$, which can be computed in $O(Bmd)$ time and $O(\min\{Bm,\,Bd+md\})$ memory.
(ii) Mean-centering makes $K$ invariant to global feature shifts and mitigates the constant-component bias; this mirrors the role of kernel centering in our analysis.
(iii) The row-shift construction ensures a nonpositive sum for all non-dominant classes.
(iv) Hyperparameters $(\lambda_+,\lambda_-)$ control the within-/cross-class tradeoff; a practical default is $\lambda_+=1$ and $\lambda_-=1/{(C-1)}$ with annealing.

\noindent
In summary, the proposed construction yields a computationally friendly surrogate for the meta reweighting update $u(\theta_t)K(\cdot,X_{\mathrm{clean}})u^v(\theta_t)$. This surrogate retains the essential signed, similarity-weighted aggregation motivated theoretically, while replacing the NTK with a centered Gram matrix and enforcing a multiclass similarity margin via row shifting and label-aware scaling. Computational details are provided in Algorithm \ref{alg:FBR}.

\textbf{Distribution of weight derivatives.}
On CIFAR-10 with $40\%$ symmetric noise, we run Algorithm~\ref{alg:FBR} and record throughout training the \emph{negative} weight derivatives
$-\partial \mathcal{L}_{\mathrm{val}}/\partial w_i$ (the instantaneous meta–update directions for $w_i$).
The empirical distribution is shown in Fig.~\ref{fig:weight_direction_final}.
Consistent with the theory, the two populations are clearly separated: most clean examples yield positive values (driving $w_i$ upward), whereas most noisy examples yield negative values (driving $w_i$ downward).
This discriminative signal underlies the subsequent filtering effect; the terminal weight histogram in Fig.~\ref{fig:weight_final} further affirms this separation.

\section{Experiments}
\label{sec:exp}
In this section, we empirically evaluate the effectiveness of our algorithm on both synthetic (i.e. CIFAR-10 and CIFAR-100) and realistic (i.e., CIFAR-N and Clothing-1M) label noise.

\noindent \textbf{Experimental setup.}
We adhere to the evaluation setup of \textsc{FINE}~\citep{kim2021fine}.

\emph{Benchmarks and noise models.}
Following standard practice, we consider two synthetic label-noise regimes on CIFAR-10/100: (i) \emph{symmetric} noise, where each label is independently flipped to a uniformly random incorrect class at a prescribed rate; and (ii) \emph{asymmetric} noise, with class-dependent flips. For CIFAR-10 we use the canonical mapping \textsc{TRUCK}$\rightarrow$\textsc{AUTOMOBILE}, \textsc{BIRD}$\rightarrow$\textsc{AIRPLANE}, \textsc{DEER}$\rightarrow$\textsc{HORSE}, \textsc{CAT}$\leftrightarrow$\textsc{DOG}. For CIFAR-100, the 100 classes are grouped into 20 super-classes (size 5) and labels are shifted cyclically within each super-class. We also evaluate on the real-world Clothing1M dataset~\citep{xiao2015learning} with naturally corrupted labels (1M images, 14 categories), which provides 50k/14k/10k verified-clean splits for training/validation/testing. As in~\citep{kim2021fine}, instead of using the 50k clean training split, we construct a pseudo-balanced training set of 120k images from the noisy pool and report accuracy on the 10k clean test set. 
For every dataset, we reserve a clean meta subset of 2{,}000 examples for weight updates. On CIFAR-10/100 this subset is drawn from the original clean labels; on Clothing1M it is drawn from the provided 14k clean validation set. 
We compare the proposed method with the following sample selection/reweighting baselines: (1) Bootstrap \citep{reed2014training}, (2) Forward \citep{patrini2017making}, (3) Co-teaching \citep{han2018co}, (4) Co-teaching++ \citep{yu2019does}, (5) TopoFilter \citep{wu2020topological}, (6) CRUST \citep{mirzasoleiman2020coresets}, (7)BHN \citep{yu2023delving}, (8) RENT \citep{bae2024dirichlet}. In the appendix, we expand our analysis to include the CIFAR-10/100-N dataset~\citep{wei2022learning}.

\emph{Results overview.}
Table~\ref{tab:cifar_accuracy} reports test accuracies across noise rates on CIFAR-10/100. On CIFAR-10 with 20\% symmetric noise our method attains 92.3\%: +1.3\% gain over the strongest prior (FINE at 91.0\%). At 50\% symmetric noise we obtain 87.0\%, which matches the performance of CRUST (87.0\%) and is within the margin error compared to FINE (87.3\%). At 40\% asymmetric noise we reach 90.6\% accuracy, improving on FINE (89.5\%) by +1.1\%. Gains are larger on CIFAR-100: at 20\%/50\% symmetric noise we achieve 73.4\% and 65.4\% respectively, exceeding FINE (70.3\%, 64.2\%) 3.0\% and 1.2\%, respectively. With 40\% asymmetric noise we observe 73.2\% accuracy, a more than 11\% improvement over the best prior report (FINE at 61.7\%). These trends indicate that our selection mechanism is especially effective in the class-dependent corruption regime and for the fine-grained CIFAR-100 label space, where overfitting to noise is most severe. On Clothing1M, our approach achieves 74.16\% top-1 accuracy, surpassing the baselines.
Empirically, the procedure retains informative clean instances while suppressing misleading ones, mitigating overfitting to noisy labels and yielding superior overall performance. 
In Table~\ref{tab:cifar_n} in Appendix~\ref{sec:appdxCIFARN}, we report test accuracies for the CIFAR-N datasets. On CIFAR-10N Aggre our method attains 92.3\% accuracy: a 0.9\% improvement over the strongest prior (JoCoR at 91.4\%). On CIFAR-10N Worst, our method attains 85.6\% accuracy: a 1.8\% improvement over the best performing prior (Co-Teaching at 83.3\%). Similar to the synthetic benchmarks, the improvement is more pronounced on CIFAR-100N (a +3.5\% improvement over CORES).

We additionally compare with Meta-Weight-Net (MW-Net) \citep{shu2019meta} on CIFAR-10/100. Table~\ref{tab:meta} in Appendix~\ref{sec:exp_detail} demonstrates the advantages of neural-tangent features and penultimate-layer features. In contrast, MW-Net shows marked overfitting in the presence of label noise.
\begin{table}[t]
\centering
\small
\setlength{\tabcolsep}{6pt}
\caption{Results on CIFAR-10/100 with symmetric and asymmetric label noise.}
\begin{tabular}{lcccccc}
\toprule
\textbf{Dataset}
  & \multicolumn{3}{c}{\textbf{CIFAR-10}}
  & \multicolumn{3}{c}{\textbf{CIFAR-100}} \\
\textbf{Noisy Type}
  & \multicolumn{2}{c}{Sym} & \multicolumn{1}{c}{Asym}
  & \multicolumn{2}{c}{Sym} & Asym \\
\textbf{Noise Ratio} & 20 & 50 & 40 & 20 & 50 & 40 \\
\midrule
Standard            & \(87.0 \pm 0.1\) & \(78.2 \pm 0.8\) & \(85.0 \pm 0.0\) & \(58.7 \pm 0.3\) & \(42.5 \pm 0.3\) & \(42.7 \pm 0.6\) \\
Bootstrap       & \(86.2 \pm 0.2\) & --               & \(81.2 \pm 1.5\) & \(58.3 \pm 0.2\) & --               & \(45.1 \pm 0.6\) \\
Forward         & \(88.0 \pm 0.4\) & --               & \(83.6 \pm 0.6\) & \(39.2 \pm 2.6\) & --               & \(34.4 \pm 1.9\) \\
Co-teaching    & \(89.3 \pm 0.3\) & \(83.3 \pm 0.6\) & \(88.4 \pm 2.8\) & \(63.4 \pm 0.4\) & \(49.1 \pm 0.4\) & \(47.7 \pm 1.2\) \\
Co-teaching+    & \(89.1 \pm 0.5\) & \(84.9 \pm 0.4\) & \(86.5 \pm 1.2\) & \(59.2 \pm 0.4\) & \(47.1 \pm 0.3\) & \(44.7 \pm 0.6\) \\
TopoFilter      & \(90.4 \pm 0.2\) & \(86.8 \pm 0.3\) & \(87.5 \pm 0.4\) & \(66.9 \pm 0.4\) & \(53.4 \pm 1.8\) & \(56.6 \pm 0.5\) \\
CRUST           & \(89.4 \pm 0.2\) & \(87.0 \pm 0.1\) & \(82.4 \pm 0.0\) & \(69.3 \pm 0.2\) & \(62.3 \pm 0.2\) & \(56.1 \pm 0.5\) \\
FINE                & \(91.0 \pm 0.1\) & \(\bm{87.3 \pm 0.2}\) & \(89.5 \pm 0.1\) & \(70.3 \pm 0.2\) & \(64.2 \pm 0.5\) & \(61.7 \pm 1.0\) \\
BHN                 & \(85.7 \pm 0.2\)              & \(85.4 \pm 0.1\)              & \(86.4 \pm 0.9\)              & \(64 .1 \pm 0.3\)              & \(48 .2 \pm 0.7\)              & \(53.4 \pm 0.9\)              \\
RENT                & \(79.8 \pm 0.2\)              &  \(66.8 \pm 0.6\)               & \(78.4 \pm 0.3\)              & \(48.4 \pm 0.6\)              & \(37.2 \pm 1.2\)              & \(37.8 \pm 1.4\)              \\
\midrule
\textbf{Ours (FBR)}       & \(\bm{92.3\pm0.2}\)     & \({87.0\pm0.1}\)     & \(\bm{90.6\pm 0.4}\)     &  \(73.4\pm0.2\)      & \(\bm{65.4\pm0.4}\)      & \(\bm{73.2\pm0.2}\)     \\
\textbf{Ours (NTK)}       & \(91.4 \pm 0.2\)     & \(86.4 \pm 0.2\)     & \(89.7 \pm 0.4\)     &  \(\bm{73.6 \pm 0.2}\)      & \(\bm{65.4\pm 0.5}\)      & \(73.1 \pm 0.4\)     \\
\bottomrule
\end{tabular}
\label{tab:cifar_accuracy}
\end{table}

\begin{table}[t]
  \centering
  \caption{Test accuracy on Clothing1M dataset.}
  \begin{tabular}{lccccccccc}
    \toprule
    Method & Standard  & SCE & ELR & CORES$^{2}$ & FINE & BHN & RENT & Ours \\
    \midrule
    Accuracy & 68.94  & 71.02 & 72.87 & 73.24 & 72.91 & $73.27$ & $70.1$ &$\bm{74.16} $ \\ 
    \bottomrule
  \end{tabular}
  \label{tab:clothing_accuracy}
\end{table}

\section{Conclusions}
We analyze the training dynamics of meta-reweighting under noisy labels when the reweighting step size scales inversely with the classifier training step size, and develop a theory explaining why the procedure separates clean from noisy data. Under standard over-parameterization and stable step sizes, training follows a three-phase trajectory—early alignment, filtering with weight polarization and clean-subset convergence, and post-filtering susceptibility. We further provide empirical evidence that corroborates both the predicted dynamics and the underlying assumptions.

Building on these insights, we proposed a computationally light surrogate for the bilevel update that preserves the same mechanism: it keeps the signed, similarity-weighted aggregation and mean-centers the similarity Gram matrix to remove global bias. Across synthetic and real-world datasets, the surrogate consistently improves test accuracy over strong reweighting/selection baselines.

\newpage


\bibliography{main}
\bibliographystyle{iclr2026/iclr2026_conference}

\appendix
\section{Previous Studies}
\label{sec:related}
\subsection{Theoretical Works for Meta Learning}
\paragraph{Convergence analysis.} 
For gradient-based meta-learning (GBML) such as MAML, several works establish convergence to an $\varepsilon$–first-order stationary point under appropriate task/data batch sizes and step-size choices \citep{fallah2020convergence,ji2022theoretical}. Moreover, using first-order updates that drop Hessian terms may forfeit these guarantees in general, while Hessian-free variants can recover them \citep{fallah2020convergence}. Multi-step MAML admits provable convergence in both resampling (fresh data per inner step) and finite-sum (fixed per-task dataset) regimes, with theory recommending inner-loop step sizes that shrink with the number of inner steps to control nested-gradient bias and variance \citep{ji2022theoretical}. 

\paragraph{Generalization and benign overfitting.} 
Stability- and distributional-analyses yield task-level generalization bounds: for \emph{recurring} test tasks (drawn from those seen in training), the excess risk decreases with both the number of tasks and per-task samples; for \emph{unseen} tasks, the bound degrades with the discrepancy between the train-task and test-task distributions (e.g., via total-variation or related divergences) \citep{fallah2021generalization}. In over-parameterized meta-linear settings, recent results show that GBML can exhibit \emph{benign overfitting}: the excess risk decomposes into cross-task and within-task variance plus bias terms, identifying regimes where interpolation does not harm—and can even improve—generalization; the theory highlights the role of task heterogeneity as well as inner-loop step size and regularization in maintaining this benign behavior \citep{chen2022understanding}.

\subsection{Noisy Label Learning}
\paragraph{Noisy-label learning (NLL).}
Deep networks easily memorize label noise, hurting generalization even when training error goes to zero \citep{zhang2021understanding}. Prior work can be organized into three main directions—label correction, sample selection, and sample re-weighting—along with surveys that systematize these trends \citep{song2022learning}.

\emph{Sample re-weighting.}
Re-weighting adjusts per-sample importance to mitigate corrupted labels (and other defects). Classical loss re-weighting \citep{liu2015classification} is complemented by meta-learning approaches that learn example weights from a small trusted set \citep{ren2018learning} or learn a parametric weighting function (Meta-Weight-Net) \citep{shu2019meta}; these typically outperform heuristic weighting under label noise \citep{ren2018learning,shu2019meta}. Other strategies infer Bayesian latent weights \citep{wang2017robust}, learn curricula via a mentor network \citep{jiang2018mentornet}, or adopt self-paced updates based on loss \citep{kumar2010self}; analogous ideas appear in adversarial robustness \citep{holtz2022learning}.

\emph{Sample selection.}
Sample selection can be regarded as a special case of sample reweighting. These methods filter or down-weight likely-noisy instances during training. The small-loss principle assumes clean samples have lower losses early on \citep{han2018co}, but may bias toward easy examples. To reduce confirmation bias and over-pruning, variants select by peer disagreement \citep{yu2019does}, co-regularize two networks with agreement (e.g., KL) \citep{wei2020combating}, or adopt dynamic thresholds with confidence regularization \citep{cheng2020learning}.

\emph{Label correction.}
Correction is often combined with \emph{sample selection} to improve robustness and stability, with selection gating updates to likely-clean examples and correction amending mislabeled ones. Leveraging the early-learning effect—models correctly predict a subset of mislabeled samples early in training \citep{liu2020early}—labels are corrected via (i) joint optimization of network parameters and (soft) labels \citep{tanaka2018joint} or (ii) probabilistic label modeling \citep{yi2019probabilistic}. Dual-network schemes stabilize correction (Co-Teaching exchanges small-loss examples \citep{han2018co}; SELFIE replaces a subset using historical predictions \citep{song2019selfie}). Hybrid methods combine correction with \emph{sample selection}, e.g., DivideMix uses a two-component mixture to split clean/noisy data and applies semi-supervised training with Mixup \citep{li2020dividemix,zhang2017mixup}, which helps gate updates and curb confirmation bias. Additional criteria/schedules include likelihood-ratio relabeling \citep{zheng2020error} and stage-/confidence-aware regularization (ProSelfLC) \citep{wang2021proselflc}.

\section{Proof of Theorem \ref{thm: main}}
\label{sec:proof}
\subsection{Early Phase}
\label{sec:early}
To prove Theorem \ref{thm: main}, we first state the NTK result of wide neural networks. 

\begin{definition}
    Let the linearized neural network of $f_{\theta_t}(x)$ be defined as  $f_{\operatorname{lin}}^{(t)}(\boldsymbol{x}):=f^{(0)}(\boldsymbol{x})+\left\langle\theta^{(t)}-\theta^{(0)}, \nabla_\theta f^{(0)}(\boldsymbol{x})\right\rangle$.
\end{definition}

\begin{theorem}
\label{thm:lin_approx}
    Let $f_{\operatorname{lin}}^{(t)}$ be its linearized neural network trained by the same reweighting factors. Under the same settings as in Theorem \ref{thm: main}, for any fixed finite $T>0$, with probability at least $1-\delta$ over random initialization we have 
    $$
    \left|f_{\operatorname{lin}}^{(t)}(\boldsymbol{x})-f^{(t)}(\boldsymbol{x})\right| \leq
      C \eta t B^{t-1} \tilde{d}^{-1 / 4}.
    $$
\end{theorem}

Begin with the following Lemma:
\begin{lemma}
\label{lem:small_ball}
    There exist constants $M>0$ such that for any $\delta>0$, there exist $R_0>0, \tilde{D}>0$ and $B>1$ such that for any $\tilde{d} \geq \tilde{D}$, the following hold with probability at least $(1-\delta)$ over random initialization and with learning rate $\eta$ :
For all $t \leq T_1$, there is

\begin{align}
\left\|u\left(\theta^{(t)}\right)\right\|_2 & \leq B^t R_0 \label{eq:ut_bound} \\
\sum_{j=1}^t\left\|\theta^{(j)}-\theta^{(j-1)}\right\|_2 & \leq \eta M R_0 \sum_{j=1}^t B^{j-1} < \eta^* \frac{M B^{T_1} R_0}{B-1} \label{eq:sum_bound}
\end{align}

\end{lemma}
\begin{proof}
 To prove Lemma \ref{lem:small_ball}, we need the following result:
    \begin{lemma}[Lemma 13 in~\cite{zhai2022understanding}]
    \label{lem:lip}
        Under Assumption \ref{assum:activation}, there is a constant $M>0$ such that for any $C_0>0$ and any $\delta>0$, there exists a $\tilde{D}$ such that: If $\tilde{d} \geq \tilde{D}$, then with probability at least $(1-\delta)$ over random initialization, for any $\boldsymbol{x}$ such that $\|\boldsymbol{x}\|_2 \leq 1$,

\begin{subequations}
\label{ineq:lip}
\begin{align}
\left\|\nabla_\theta f(\boldsymbol{x} ; \theta)-\nabla_\theta f(\boldsymbol{x} ; \tilde{\theta})\right\|_2 & \leq \frac{M}{\sqrt[4]{\tilde{d}}}\|\theta-\tilde{\theta}\|_2 \\
\left\|\nabla_\theta f(\boldsymbol{x} ; \theta)\right\|_2 & \leq M \\
\|J(\theta)-J(\tilde{\theta})\|_F & \leq \frac{M}{\sqrt[4]{\tilde{d}}}\|\theta-\tilde{\theta}\|_2, \quad \forall \theta, \tilde{\theta} \in B\left(\theta^{(0)}, C_0\right) \\
\|J(\theta)\|_F & \leq M
\end{align}
\end{subequations}
for $\forall \theta, \tilde{\theta} \in B\left(\theta^{(0)}, C_0\right)$, 
where $B\left(\theta^{(0)}, R\right)=\left\{\theta:\left\|\theta-\theta^{(0)}\right\|_2<R\right\}$.
    \end{lemma}

Note that for any $\boldsymbol{x}, f^{(0)}(\boldsymbol{x})=0$. Thus, for any $\delta>0$, there exists a constant $R_0$ such that with probability at least $(1-\delta / 3)$ over random initialization,

\begin{align}
\label{ineq:init_u}
\left\|u\left(\theta^{(0)}\right)\right\|_2<R_0
\end{align}

By the NTK result:
\begin{lemma}
    If $\sigma$ is Lipschitz and $d_l \rightarrow \infty$ for $l=1, \cdots$, L sequentially, then $K^{(0)}\left(\boldsymbol{x}, \boldsymbol{x}^{\prime}\right)$ converges in probability to a non-degenerate deterministic limiting kernel $K\left(\boldsymbol{x}, \boldsymbol{x}^{\prime}\right)$.
\end{lemma}

Let $M$ be the constant in Lemma \ref{lem:lip}. Let $B=1+\eta^* M^2$, and $C_0=\eta^*\frac{M B^{T_1} R_0}{B-1}$. By Lemma \ref{lem:lip}, there exists $D_1>0$ such that with probability at least $(1-\delta / 3)$, for any $\tilde{d} \geq D_1$,  \ref{ineq:lip} is true for all $\theta, \tilde{\theta} \in B\left(\theta^{(0)}, C_0\right)$.

By union bound, with probability at least $1-\delta$, \ref{ineq:lip}, \ref{ineq:init_u} hold. We prove \ref{eq:ut_bound} and \ref{eq:sum_bound} by induction. 

\begin{align*}
\left\|\theta^{(t+1)}-\theta^{(t)}\right\|_2 & \leq \eta\left\|J\left(\theta^{(t)}\right) W^{(t)}\right\|_2\left\|u\left(\theta^{(t)}\right)\right\|_2 \\
& \leq \eta\left\|J\left(\theta^{(t)}\right) W^{(t)}\right\|_F\left\|u\left(\theta^{(t)}\right)\right\|_2 \\
& \leq \eta\left\|J\left(\theta^{(t)}\right)\right\|_F\left\|u\left(\theta^{(t)}\right)\right\|_2 \\
&\leq \eta M B^t R_0
\end{align*}
\begin{align*}
\left\|u\left(\theta^{(t+1)}\right)\right\|_2 & =\left\|u\left(\theta^{(t+1)}\right)-u\left(\theta^{(t)}\right)+u\left(\theta^{(t)}\right)\right\|_2 \\
& =\left\|J\left(\tilde{\theta}^{(t)}\right)^{\top}\left(\theta^{(t+1)}-\theta^{(t)}\right)+u\left(\theta^{(t)}\right)\right\|_2 \\
& =\left\|-\eta J\left(\tilde{\theta}^{(t)}\right)^{\top} J\left(\theta^{(t)}\right) W^{(t)} u\left(\theta^{(t)}\right)+u\left(\theta^{(t)}\right)\right\|_2 \\
& \leq\left\|\boldsymbol{I}-\eta J\left(\tilde{\theta}^{(t)}\right)^{\top} J\left(\theta^{(t)}\right) W^{(t)}\right\|_2\left\|u\left(\theta^{(t)}\right)\right\|_2 \\
& \leq\left(1+\left\|\eta J\left(\tilde{\theta}^{(t)}\right)^{\top} J\left(\theta^{(t)}\right) W^{(t)}\right\|_2\right)\left\|u\left(\theta^{(t)}\right)\right\|_2 \\
& \leq\left(1+\eta \left\|J\left(\tilde{\theta}^{(t)}\right)\right\|_F\left\|J\left(\theta^{(t)}\right)\right\|_F\right)\left\|u\left(\theta^{(t)}\right)\right\|_2 \\
& \leq\left(1+\eta  M^2\right)\left\|u\left(\theta^{(t)}\right)\right\|_2 \leq B^{t+1} R_0
\end{align*}
\end{proof}
Now start to prove Theorem \ref{thm:lin_approx}.
\begin{proof}
    Denote $\Delta_t=u_{\operatorname{lin}}\left(\theta^{(t)}\right)-u\left(\theta^{(t)}\right)$. Then
    \begin{equation}
\left\{\begin{aligned}
u_{\operatorname{lin}}\left(\theta^{(t+1)}\right)-u_{\operatorname{lin}}\left(\theta^{(t)}\right) & =-\eta J\left(\theta^{(0)}\right)^{\top} J\left(\theta^{(0)}\right) W^{(t)} u_{\operatorname{lin}}\left(\theta^{(t)}\right) \\
u\left(\theta^{(t+1)}\right)-u\left(\theta^{(t)}\right) & =-\eta J\left(\tilde{\theta}^{(t)}\right)^{\top} J\left(\theta^{(t)}\right) W^{(t)} u\left(\theta^{(t)}\right)
\end{aligned}\right.
\end{equation}
Thus 
\begin{align*}
\Delta_{t+1}-\Delta_t= & \eta\left[J\left(\tilde{\theta}^{(t)}\right)^{\top} J\left(\theta^{(t)}\right)-J\left(\theta^{(0)}\right)^{\top} J\left(\theta^{(0)}\right)\right] W^{(t)} u\left(\theta^{(t)}\right) \\
& -\eta J\left(\theta^{(0)}\right)^{\top} J\left(\theta^{(0)}\right) W^{(t)} \Delta_t
\end{align*}
By Lemma \ref{lem:lip}, 

\begin{align*}
& \left\|J\left(\tilde{\theta}^{(t)}\right)^{\top} J\left(\theta^{(t)}\right)-J\left(\theta^{(0)}\right)^{\top} J\left(\theta^{(0)}\right)\right\|_F \\
\leq & \left\|\left(J\left(\tilde{\theta}^{(t)}\right)-J\left(\theta^{(0)}\right)\right)^{\top} J\left(\theta^{(t)}\right)\right\|_F+\left\|J\left(\theta^{(0)}\right)^{\top}\left(J\left(\theta^{(t)}\right)-J\left(\theta^{(0)}\right)\right)\right\|_F \\
\leq & 2 M^2 C_0 \tilde{d}^{-1 / 4}
\end{align*}
Then 
\begin{align*}
&\left\|\Delta_{t+1}\right\|_2 \\
& \leq\left\|\left[\boldsymbol{I}-\eta J\left(\theta^{(0)}\right)^{\top} J\left(\theta^{(0)}\right) W^{(t)}\right] \Delta_t\right\|_2+\left\|\eta\left[J\left(\tilde{\theta}^{(t)}\right)^{\top} J\left(\theta^{(t)}\right)-J\left(\theta^{(0)}\right)^{\top} J\left(\theta^{(0)}\right)\right] W^{(t)} u\left(\theta^{(t)}\right)\right\|_2 \\
& \leq\left\|\boldsymbol{I}-\eta J\left(\theta^{(0)}\right)^{\top} J\left(\theta^{(0)}\right) W^{(t)}\right\|_F\left\|\Delta_t\right\|_2+\eta \left\|J\left(\tilde{\theta}^{(t)}\right)^{\top} J\left(\theta^{(t)}\right)-J\left(\theta^{(0)}\right)^{\top} J\left(\theta^{(0)}\right)\right\|_F\left\|u\left(\theta^{(t)}\right)\right\|_2 \\
& \leq\left(1+\eta  M^2\right)\left\|\Delta_t\right\|_2+2 \eta  M^2 C_0 B^t R_0 \tilde{d}^{-1 / 4} \\
& \leq B\left\|\Delta_t\right\|_2+2 \eta M^2 C_0 B^t R_0 \tilde{d}^{-1 / 4}
\end{align*}

Therefore, 
$$
B^{-(t+1)}\left\|\Delta_{t+1}\right\|_2 \leq B^{-t}\left\|\Delta_t\right\|_2+2 \eta  M^2 C_0 B^{-1} R_0 \tilde{d}^{-1 / 4}.
$$

Since $\Delta_0 =0$,
we have 
$$ \left\|\Delta_t\right\|_2 \leq 2 t \eta  M^2 C_0 B^{t-1} R_0 \tilde{d}^{-1 / 4}. $$
\end{proof}

Proof of Theorem \ref{thm: main}:

\begin{proof}

We first calculate $ \nabla_w\sum_{j=1}^m l^v_j(\theta_t-\eta \nabla \sum_{i=1}^n  w_t^i l_{i}(\theta_t))$.

By a direct computation, we notice

\begin{align*}
    &\frac{\partial}{\partial w^k} l^v_j(\theta_t-\eta \nabla \sum_{i=1}^n  w_t^i l_{i}(\theta_t)) \\
    &=  (f^v_j(\theta_t-\eta \nabla \sum_{i=1}^n  w_t^i l_{i}(\theta_t)) - y^v_j) \frac{\partial f_j^v(\theta_t-\eta \nabla \sum_{i=1}^n  w_t^i l_{i}(\theta_t))}{\partial w^k}
\end{align*}
where 
    \begin{align*}
    \frac{d f^v_j(\theta_t-\eta \nabla_{\theta}  w_t^T l(\theta_t))}{d w^k} &= \langle \nabla_{\theta} f^v_j(\theta_t-\eta \nabla_{\theta}  w_t^T l(\theta_t)), \frac{d}{d w^k} ( \theta_t-\eta \nabla_{\theta}  w_t^T l(\theta_t))\rangle  \\
    &=
    \langle \nabla_{\theta} f^v_j(\theta_t-\eta \nabla_{\theta}  w_t^T l(\theta_t)),  -\eta \nabla_{\theta} l_k(\theta_t)\rangle \\
    &=
    \langle \nabla_{\theta} f^v_j(\theta_t-\eta \nabla_{\theta}  w_t^T l(\theta_t)),  -\eta (f_k(\theta_t) - y_k) \nabla f_k(\theta_t)\rangle
\end{align*}
Therefore, 
\begin{align*}
    &\frac{\partial}{\partial w^k} l^v_j(\theta_t-\eta \nabla \sum_{i=1}^n  w_t^i l_{i}(\theta_t)) \\
    &= -\eta u^v_j(\theta_t-\eta \nabla \sum_{i=1}^n  w_t^i l_{i}(\theta_t))  u_k(\theta_t) \langle \nabla_{\theta} f^v_j(\theta_t-\eta \nabla_{\theta}  w_t^T l(\theta_t)),   \nabla_\theta f_k(\theta_t)\rangle
\end{align*}
and consequently
\begin{align*}
    &\sum_{j=1}^m\frac{\partial}{\partial w^k} l^v_j(\theta_t-\eta \nabla \sum_{i=1}^n  w_t^i l_{i}(\theta_t)) \\
    &= -\eta   \langle \sum_{j=1}^m u^v_j(\theta_t-\eta \nabla \sum_{i=1}^n  w_t^i l_{i}(\theta_t)) \nabla_{\theta} f^v_j(\theta_t-\eta \nabla_{\theta}  w_t^T l(\theta_t)),  u_k(\theta_t) \nabla_\theta f_k(\theta_t)\rangle.
\end{align*}
So the weight gradient can be computed as
\begin{align*}
    &\nabla_w\sum_{j=1}^m l^v_j(\theta_t-\eta \nabla \sum_{i=1}^n  w_t^i l_{i}(\theta_t)) \\
    &= -\eta \Big(J(\theta_t) U(\theta_t) \Big)^T \Big( \sum_{j=1}^m u^v_j(\theta_t-\eta \nabla \sum_{i=1}^n  w_t^i l_{i}(\theta_t)) \nabla_{\theta} f^v_j(\theta_t-\eta \nabla_{\theta}  w_t^T l(\theta_t)) \Big) \\
    &= -\eta \Big(J(\theta_t) U(\theta_t) \Big)^T \Big( J^v(\theta_t-\eta \nabla_{\theta}  w_t^T l(\theta_t)) u^v(\theta_t-\eta \nabla_{\theta}  w_t^T l(\theta_t))\Big) \\
    &= -\eta U(\theta_t) J(\theta_t)^T J^v(\theta_t-\eta \nabla_{\theta}  w_t^T l(\theta_t)) u^v(\theta_t-\eta \nabla_{\theta}  w_t^T l(\theta_t)) \\
    &=  -\eta U(\theta_t) J(\theta_t)^T \Big(J^v(\theta_t) + J^v(\theta_t-\eta \nabla_{\theta}  w_t^T l(\theta_t)) - J^v(\theta_t)\Big) \Big(u^v(\theta_t) + u^v(\theta_t-\eta \nabla_{\theta}  w_t^T l(\theta_t)) - u^v(\theta_t)\Big) \\
    &\approx -\eta U(\theta_t) J(\theta_t)^T J^v(\theta_t) u^v(\theta_t) + O(\eta^2 error),
\end{align*}
where $U(\theta_t)$ denote the diagonal matrix from $u(\theta_t)$. To estimate the approximation error in the last step, we observe
$$ \|J^v(\theta_t-\eta \nabla_{\theta}  w_t^T l(\theta_t)) - J^v(\theta_t)\|_F \leq \frac{M}{\sqrt[4]{\tilde{d}}}\|\eta \nabla_{\theta}  w_t^T l(\theta_t)\|_2 \leq \frac{M^2}{\sqrt[4]{\tilde{d}}}\eta B^tR_0 $$
and 
$$ \|u^v(\theta_t-\eta \nabla_{\theta}  w_t^T l(\theta_t)) - u^v(\theta_t)\|_2 = \| \eta J^v(\tilde{\theta})^T \nabla_{\theta}  w_t^T l(\theta_t) \|_2 \leq \eta M^2B^tR_0.$$
As a consequence, 
\begin{align*}
    & \|\nabla_w\sum_{j=1}^m l^v_j(\theta_t-\eta \nabla \sum_{i=1}^n  w_t^i l_{i}(\theta_t)) + \eta U(\theta_t) J(\theta_t)^T J^v(\theta_t) u^v(\theta_t)\|_2 \\
    &= \| \eta U(\theta_t) J(\theta_t)^T \Big(J^v(\theta_t) + J^v(\theta_t-\eta \nabla_{\theta}  w_t^T l(\theta_t)) - J^v(\theta_t)\Big) \Big(u^v(\theta_t) + u^v(\theta_t-\eta \nabla_{\theta}  w_t^T l(\theta_t)) - u^v(\theta_t)\Big) \|_2 \\
    &= \| \eta U(\theta_t) J(\theta_t)^T \Big( J^v(\theta_t-\eta \nabla_{\theta}  w_t^T l(\theta_t)) - J^v(\theta_t)\Big) u^v(\theta_t) \|_2  \\
    &+ \| \eta U(\theta_t) J(\theta_t)^T J^v(\theta_t) \Big(u^v(\theta_t-\eta \nabla_{\theta}  w_t^T l(\theta_t)) - u^v(\theta_t)\Big) \|_2\\
    &+ \| \eta U(\theta_t) J(\theta_t)^T \Big(J^v(\theta_t-\eta \nabla_{\theta}  w_t^T l(\theta_t)) - J^v(\theta_t)\Big) \Big( u^v(\theta_t-\eta \nabla_{\theta}  w_t^T l(\theta_t)) - u^v(\theta_t)\Big) \|_2\\
    &\lesssim \eta B^{t} R_0M \Big(\frac{M^2}{\sqrt[4]{\tilde{d}}} \eta B^{2t}R_0^2 + \eta M^2 B^t R_0 + \frac{M^4}{\sqrt[4]{\tilde{d}}} \eta^2 B^{3t}R_0^3\Big) \\
    &\lesssim \eta^2 B^{2t} R_0^2M^3.
\end{align*}
Together with  
\begin{align*}
&\|\bm K^{(t)} - \bm K^{(0)}\|_2 \\
=&\left\|J\left({\theta}^{(t)}\right)^{\top} J\left(\theta^{(t)}\right)-J\left(\theta^{(0)}\right)^{\top} J\left(\theta^{(0)}\right)\right\|_2 \\
\leq & \left\|\left(J\left({\theta}^{(t)}\right)-J\left(\theta^{(0)}\right)\right)^{\top} J\left(\theta^{(t)}\right)\right\|_2+\left\|J\left(\theta^{(0)}\right)^{\top}\left(J\left(\theta^{(t)}\right)-J\left(\theta^{(0)}\right)\right)\right\|_2 \\
\leq & 2 M^2 C_0 \tilde{d}^{-1 / 4},
\end{align*}
we get 
\begin{align*}
    &\|w_{t+1} - \Big(  w_t + \eta \alpha U(\theta_t) K^{(0)}(X, X_{clean})u^v(\theta_t) \Big)\|_2 \\
    & \lesssim \eta^2 \alpha B^{2t} R_0^2M^3 + \eta \alpha\| U(\theta_t) K^{(t)}(X, X_{clean})u^v(\theta_t) - U(\theta_t) K^{(0)}(X, X_{clean})u^v(\theta_t) \|_2 \\
    & \lesssim \alpha \eta^2 B^{2t} R_0^2M^3 +\alpha \eta B^{2t}R_0^2M^2C_0\tilde{d}^{-1 / 4}.
\end{align*}
Now we estimate the difference of replacing $\bm u$ by $\bm y$:
\begin{align*}
    &\| U(\theta_t) K^{(0)}(X_{tr}, X_{val})u^v(\theta_t) - U_0 K^{(0)}(X_{tr}, X_{val})u^v_0  \|_\infty \\
    & = \| \big(U_t-U_0\big) K^{(0)}(X_{tr}, X_{val})u^v_0\|_\infty + \| U_0 K^{(0)}(X_{tr}, X_{val})\big( u^v_t-u^v_0\big)\|_\infty \\
     &+ \|\big(U_t-U_0\big) K^{(0)}(X_{tr}, X_{val})\big(u^v_t-u^v_0\big) \|_\infty.
\end{align*}
Since $\bm u_t - \bm u_0$ can be bounded via:
\begin{align*}
    \| \bm u_{t} - \bm u_{0}  \|_\infty &= \|\bm u_{t}- \bm u^{(t)}_{lin} + \bm u^{(t)}_{lin} -\bm u_{0}\|_\infty \\
    &\leq  \|\bm u_{t}- \bm u^{(t)}_{lin} \|_\infty + \| \bm u^{(t)}_{lin} -\bm u_{0}\|_\infty\\
    &\leq C \eta t B^{t-1} \tilde{d}^{-1 / 4} + \| \bm u^{(t)}_{lin} -\bm u_{0}\|_2\\
    &= C \eta t B^{t-1} \tilde{d}^{-1 / 4} + \| f_{\mathrm{lin}}^{(t)}(X) 
    -f_{\mathrm{lin}}^{(0)}(X) \|_2 \\
    &= C \eta t B^{t-1} \tilde{d}^{-1 / 4} + \|\nabla_\theta f^{(0)}(X)^T  \bigl(\theta^{(t)}-\theta^{(0)}\bigr)\|_2 \\
    &\leq C \eta t B^{t-1} \tilde{d}^{-1 / 4} + \sqrt{\lambda_{max}} \| \theta^{(t)} - \theta^{(0)}\|_2\\
    &\leq C \eta t B^{t-1} \tilde{d}^{-1 / 4} +\sqrt{\lambda_{max}} \eta C_0.
\end{align*}

Combine the previous error bounds, we get 
\begin{align*}
    &\| w_{t+1} - \big( w_t + \eta\alpha U_0 K^{(0)}(X_{tr}, X_{val})u^v_0 \big)\|_\infty \\
    &\leq \|w_{t+1} - \big(  w_t + \eta \alpha U(\theta_t) K^{(0)}(X_{tr}, X_{val})u^v(\theta_t) \big)\|_\infty \\&+\alpha \eta\| U(\theta_t) K^{(0)}(X_{tr}, X_{val})u^v(\theta_t) - U_0 K^{(0)}(X_{tr}, X_{val})u^v_0  \|_\infty \\
    &\lesssim  \alpha \eta^2 B^{2t} R_0^2M^3 +\alpha \eta B^{2t}R_0^2M^2C_0\tilde{d}^{-1 / 4} +  C \alpha \eta^2 \lambda_{max}\sqrt{m}  t B^{t-1} \tilde{d}^{-1 / 4} + \alpha \eta^2 \lambda_{max}^{3/2} C_0 \\
    &\lesssim \alpha \eta^2 B^{2t} + \alpha \eta B^{2t} C_0 \tilde{d}^{-1 / 4} + \alpha \eta^2 \lambda_{max}\sqrt{m}  t B^{t-1} \tilde{d}^{-1 / 4} + \alpha \eta^2 \lambda_{max}^{3/2} C_0 \\
    &\lesssim \alpha \eta^2 \lambda_{max}^{3/2} B^{2T_1} + \alpha \eta^2 \lambda_{max}\sqrt{m}  T_1 B^{T_1} \tilde{d}^{-1 / 4} +  \alpha \eta B^{3T_1} \tilde{d}^{-1 / 4}.
\end{align*}
Thus, together with Assumption \ref{assum:kernel_signs}, under a suitable choice of $\eta^*$ and $\tilde D$, for any $\eta<\eta^*$ and $\tilde{d}>\tilde D$, we have $w_{T_1}^{(i)}=0$ for noisy samples and $w_{T_1}^{(i)}=1$ for clean samples at $T_1= 1+(m\alpha \eta \gamma)^{-1}$.  
\end{proof}

\subsection{Filtering Phase}

We analyze the algorithm for $T_1\leq t\leq T_2$. Note that since the error bound \ref{sec:early} could increase as $t$ increases, and in order for the training error to converge, it is necessary to have a sufficiently large $t$, we need a different way to control the weights $\bm w$ while obtaining convergence of the training loss. For this section, we prove 
\begin{align*}
    w_t^i &= 
  \begin{cases}
    1,         & y_i = y^*_i,\\[3pt]
    0,         & y_i \neq y^*_i.\\[3pt]
  \end{cases}
\end{align*}
  by induction. We denote $W^*$ as the diagonal matrix formed by $w^i$. Specifically, we show
  
\begin{lemma}
    There exist constants $M>0$ such that for any $\delta>0$, there exist $R_0>0, \tilde{D}>0$ and $B>1$ such that for any $\tilde{d} \geq \tilde{D}$, the following hold with probability at least $(1-\delta)$ over random initialization when applying gradient descent with learning rate $\eta$ :
For all $T_1\leq t\leq T_2$, we have 

\begin{align}
\left\|W^*u\left(\theta^{(t)}\right)\right\|_2 & \leq \Big( 1- \frac{\eta  \lambda_{\min }}{3} \Big)^{t-T_1}B^{T_1} R_0 \label{eq:ut_bound2} \\
\sum_{j=T_1+1}^t\left\|\theta^{(j)}-\theta^{(j-1)}\right\|_2 
& \leq \eta   M B^{T_1} R_0 \sum_{j=T_1+1}^t\left(1-\frac{\eta \lambda_{\min }}{3}\right)^{j-T_1} \label{eq:sum_bound2}\\ 
&< \frac{3   M B^{T_1} R_0}{ \lambda_{\min }}  \\
w_t^i
  &= 
  \begin{cases}
    1,         & y_i = y^*_i,\\[3pt]
    0,         & y_i \neq y^*_i.\\[3pt]
  \end{cases} \label{eq: weight_stage2}
\end{align}

\end{lemma}

\begin{proof}
There exists $D_2 \geq 0$ such that for any $\tilde{d} \geq D_2$, with probability at least $(1-\delta / 3)$,
\begin{align}
    \label{ineq:init_kernel}
    \left\|\bm K-\bm K^{(0)}\right\|_F \leq \frac{\lambda_{\min }}{3}.
\end{align}

    We prove this lemma by induction. We first check the conclusions when $t=T_1$. By \ref{eq:ut_bound}, $\left\|W^*u\left(\theta^{(t)}\right)\right\|_2  \leq \|u\left(\theta^{(t)}\right)\| \leq B^{T_1}R_0$. \ref{eq:sum_bound2} holds trivially. \ref{eq: weight_stage2} holds due to the choice of $T_1$. 

    Then for $t+1$, 
    \begin{align}
        \begin{aligned}
\left\|\theta^{(t+1)}-\theta^{(t)}\right\|_2 & \leq \eta\left\|J\left(\theta^{(t)}\right) \right\|_2\left\|W^* u\left(\theta^{(t)}\right)\right\|_2 \\
&\leq \eta\left\|J\left(\theta^{(t)}\right) \right\|_F\left\|W^*  u\left(\theta^{(t)}\right)\right\|_2 \\
&\leq M \eta \left(1-\frac{\eta \lambda_{\min }}{3}\right)^{t-T_1} B^{T_1} R_0
\end{aligned}
    \end{align}

Consequently, 
\begin{align*}
\left\| W^* u\left(\theta^{(t+1)}\right)\right\|_2 & =\left\| W^* u\left(\theta^{(t+1)}\right)- W^* u\left(\theta^{(t)}\right)+ W^* u\left(\theta^{(t)}\right)\right\|_2 \\
& =\left\| W^* J\left(\tilde{\theta}^{(t)}\right)^{\top}\left(\theta^{(t+1)}-\theta^{(t)}\right)+ W^* u\left(\theta^{(t)}\right)\right\|_2 \\
& =\left\|-\eta  W^* J\left(\tilde{\theta}^{(t)}\right)^{\top} J\left(\theta^{(t)}\right) W^* u\left(\theta^{(t)}\right)+ W^* u\left(\theta^{(t)}\right)\right\|_2 \\
& \leq\left\|\boldsymbol{I}-\eta  W^* J\left(\tilde{\theta}^{(t)}\right)^{\top} J\left(\theta^{(t)}\right)  W^*\right\|_{2, W^*}\left\| W^* u\left(\theta^{(t)}\right)\right\|_2 \\
& \leq\left\|\boldsymbol{I}-\eta  W^* J\left(\tilde{\theta}^{(t)}\right)^{\top} J\left(\theta^{(t)}\right)  W^*\right\|_{2, W^*}\left(1-\frac{\eta \lambda_{\min}}{3}\right)^{t-T_1} B^{T_1}R_0,
\end{align*}
where $\|\cdot\|_{2, W^*}$ indicates the $2$-norm restricted on the image space of $W^*$. 

Since 
\begin{align*}
& \left\|\boldsymbol{I}-\eta W^* J\left(\tilde{\theta}^{(t)}\right)^{\top} J\left(\theta^{(t)}\right) W^*\right\|_{2, W^*} \\
 &\leq \|\boldsymbol{I}-\eta W^* K W^*\|_{2, W^*}+\eta\left\|W^*\left(K-K^{(0)}\right) W^*\right\|_{2, W^*}\\
&+\eta\left\|W^*\left(J\left(\theta^{(0)}\right)^{\top} J\left(\theta^{(0)}\right)-J\left(\tilde{\theta}^{(t)}\right)^{\top} J\left(\theta^{(t)}\right)\right) W^*\right\|_{2, W^*} \\
 &\leq 1-\eta \lambda_{\min }  +\eta\left\|K-K^{(0)}\right\|_F+\eta\left\|J\left(\theta^{(0)}\right)^{\top} J\left(\theta^{(0)}\right)-J\left(\tilde{\theta}^{(t)}\right)^{\top} J\left(\theta^{(t)}\right)\right\|_F \\
&\leq  1-\eta \lambda_{\min }  +\frac{\eta   \lambda_{\min }}{3}+\frac{\eta M^2}{\sqrt[4]{\tilde{d}}}\left(\left\|\theta^{(t)}-\theta^{(0)}\right\|_{2}+\left\|\tilde{\theta}^{(t)}-\theta^{(0)}\right\|_{2}\right) \\
&\leq 1-\frac{\eta   \lambda_{\min }}{3}
\end{align*}
for $\tilde{d} \geq \max \left\{D_1, D_2,\left(\frac{6 M^2 C_0}{  \lambda_{\min }}\right)^4\right\}$.

Therefore, 
$$ \left\| W^* u\left(\theta^{(t+1)}\right)\right\|_2 \leq  \left(1-\frac{\eta \lambda^{\text {min }}}{3}\right)^{t+1-T_1} B^{T_1}R_0.$$

Now we can improve our linear approximation error to be uniform on $t$.

Again, by writing 
\begin{align*}
    &\nabla_w\sum_{j=1}^m l^v_j(\theta_t-\eta \nabla \sum_{i=1}^n  w_t^i l_{i}(\theta_t)) \\
    &= -\eta U(\theta_t) J(\theta_t)^T \Big(J^v(\theta_t) + J^v(\theta_t-\eta \nabla_{\theta}  w_t^T l(\theta_t)) - J^v(\theta_t)\Big) \Big(u^v(\theta_t) + u^v(\theta_t-\eta \nabla_{\theta}  w_t^T l(\theta_t)) - u^v(\theta_t)\Big) \\
    &= -\eta U(\theta_t) \Big[ J(\theta_t)^T J^v(\theta_t) u^v(\theta_t) - J(\theta_t)^T \big(J^v(\theta_t-\eta \nabla_{\theta}  w_t^T l(\theta_t)) - J^v(\theta_t)\big) u^v(\theta_t) \\
    &- J(\theta_t)^T J^v(\theta_t)  \big(u^v(\theta_t-\eta \nabla_{\theta}  w_t^T l(\theta_t)) - u^v(\theta_t)\big) \\
    &+ J(\theta_t)^T \big(J^v(\theta_t-\eta \nabla_{\theta}  w_t^T l(\theta_t)) - J^v(\theta_t)\big)\big(u^v(\theta_t-\eta \nabla_{\theta}  w_t^T l(\theta_t)) - u^v(\theta_t)\big)\Big]
\end{align*}
and with the bounds 
$$ \|J^v(\theta_t-\eta \nabla_{\theta}  w_t^T l(\theta_t)) - J^v(\theta_t)\|_F \leq \frac{M}{\sqrt[4]{\tilde{d}}}\|\eta \nabla_{\theta}  w_t^T l(\theta_t)\|_2 \lesssim \frac{M^2}{\sqrt[4]{\tilde{d}}}\eta,$$
$$ \|u^v(\theta_t-\eta \nabla_{\theta}  w_t^T l(\theta_t)) - u^v(\theta_t)\|_2 = \| \eta J^v(\tilde{\theta})^T \nabla_{\theta}  w_t^T l(\theta_t) \|_2 \lesssim \eta M^2,$$
we have 
\begin{align*}
    & \| J(\theta_0)^T J^v(\theta_0) u^v(\theta_t) -  \Big[ J(\theta_t)^T J^v(\theta_t) u^v(\theta_t) - J(\theta_t)^T \big(J^v(\theta_t-\eta \nabla_{\theta}  w_t^T l(\theta_t)) - J^v(\theta_t)\big) u^v(\theta_t) \\
    &- J(\theta_t)^T J^v(\theta_t)  \big(u^v(\theta_t-\eta \nabla_{\theta}  w_t^T l(\theta_t)) - u^v(\theta_t)\big) \\
    &+ J(\theta_t)^T \big(J^v(\theta_t-\eta \nabla_{\theta}  w_t^T l(\theta_t)) - J^v(\theta_t)\big)\big(u^v(\theta_t-\eta \nabla_{\theta}  w_t^T l(\theta_t)) - u^v(\theta_t)\big)\Big]\|_\infty\\
    &\lesssim \| J(\theta_0)^T J^v(\theta_0) u^v(\theta_t) - J(\theta_t)^T J^v(\theta_t) u^v(\theta_t)\|_\infty + \eta M^3 \tilde{d}^{-1/4} + \eta M^4 \\
    &\lesssim \tilde{d}^{-1/4}+\eta.
\end{align*}

Finally, we observe from
$u_{t}^i = u_{t-1}^i - \eta \sum_{j=1}^n K_{ij}w_{t}^ju_{t-1}^j + O(\tilde C \tilde{d}^{-1/4})$ that for noisy sample indices $i$, (1) $w_{t}^j=0$ for noisy sample indices $j$; (2) $sgn(K_{ij}u_t^j)=-sgn(u_t^i)$ for clean indices $j$. As a result, $u_{t}^i$ keeps the sign while obtaining larger magnitude. Similarly, for clean samples $u_{t+1}^i \rightarrow B(0, r)$ monotonically after $T_1$ for some $r = O(\eta + \tilde{d}^{-1/4})$. Therefore, the noisy sample filtering keeps until some $u^v_i \lesssim \eta +\tilde d^{-1/4}$. With the following additional Assumption~\ref{assum:stable_convergence}, the noisy sample filtering phase can be extend until $\| \bm u^v\|_\infty \lesssim \eta +\tilde d^{-1/4}$. 

\begin{assumption}
\label{assum:stable_convergence}
Denote the decomposition of $\bm u^v$ as 
\[
u^v_i(t)=\sum_{j=1}^n s_{ij}\,e^{-\eta\lambda_j t}.
\]
Let $j^\ast:=\arg\min\{\lambda_j>0:\ s_{ij}\neq 0\}$ and define $
s_i^\ast:=s_{ij^\ast}, \lambda_i^\ast:=\lambda_{j^\ast}.
$
Let $t_i$ be the first time such that $|u^v_i(t_i)|\lesssim \eta+\tilde d^{-1/4}$.
Assume for $\forall t\ge t_i$, 
\[
\operatorname{sgn}\!\big(u^v_i(t)\big)=\operatorname{sgn}\!\big(s_i^\ast\big).
\]

\end{assumption}

The noise sample filtering fails for $u^v(\theta_t)$ when the weight update signal is overwhelmed by the approximation error bound, i.e. 
$$ \gamma\sum_k |u^v_k| \lesssim \tilde{d}^{-1/4}+\eta.$$

\end{proof}

\section{Proof of Proposition 1}

\begin{proof}
\textbf{Upper bound for $C$:}
Let $s:=K_m\mathbf 1_m$. Since $\mathbf 1_m^\top u^v=0$,
\[
C \;=\; -\frac{1}{m}\,s^\top u^v
    \;=\; -\frac{1}{m}\,(s-c_m\mathbf 1_m)^\top u^v .
\]
By the Cauchy--Schwarz inequality and \eqref{eq:RS-eps},
\[
|C|
\;\le\; \frac{1}{m}\,\|s-c_m\mathbf 1_m\|_2\,\|u^v\|_2
\;=\; \frac{1}{m}\cdot O_p\!\big(m^{1-\varepsilon}\big)\cdot \sqrt m
\;=\; O_p\!\big(m^{\tfrac12-\varepsilon}\big).
\]

\textbf{Lower-order growth for $S_i$:}
Assume $x_i$ is independent of $X_{\mathrm{clean}}$. Define the row-wise mean
$\mu_K(x_i):=\mathbb E_{X'\sim\mathcal D}[K(x_i,X')]$
and the centered row vector
$k_i^\circ := K(x_i,X_{\mathrm{clean}})-\mu_K(x_i)\,\mathbf 1_m^\top\in\mathbb R^m$.
Because $\mathbf 1_m^\top u^v=0$,
\[
S_i \;=\; K(x_i,X_{\mathrm{clean}})\,u^v
      \;=\; \big(k_i^\circ+\mu_K(x_i)\mathbf 1_m^\top\big)u^v
      \;=\; (k_i^\circ)^\top u^v .
\]
Let $Z_j:=K(x_i,X^{\mathrm{clean}}_j)-\mu_K(x_i)$ for $j=1,\dots,m$.
Conditional on $x_i$, the variables $Z_1,\dots,Z_m$ are i.i.d., satisfy $\mathbb E[Z_j\mid x_i]=0$ and $|Z_j|\le 2K_{\max}$, and
$\mathrm{Var}(Z_j\mid x_i)\ge \sigma^2$ by the non-degeneracy assumption. Since $S_i=\sum_{j=1}^m u^v_j Z_j$ with $|u^v_j|=1$,
\[
\mathbb E\big[S_i^2\mid x_i\big]=\sum_{j=1}^m (u^v_j)^2\,\mathrm{Var}(Z_j\mid x_i)\ \ge\ m\,\sigma^2.
\]
Moreover, boundedness yields a finite fourth moment: there exists a constant $C_4>0$ (depending only on $K_{\max}$) such that
$\mathbb E[S_i^4\mid x_i]\le C_4\,m^2 K_{\max}^4$.
Applying the Paley--Zygmund inequality to $S_i^2$ conditional on $x_i$ with parameter $\gamma=\tfrac12$ gives
\[
\mathbb P\!\left(\,|S_i|\ \ge\ \tfrac12\,\sqrt{\mathbb E[S_i^2\mid x_i]}\ \middle|\ x_i\right)
\ \ge\ (1-\tfrac12)^2\,\frac{\mathbb E[S_i^2\mid x_i]^2}{\mathbb E[S_i^4\mid x_i]}
\ \ge\ \frac{\tfrac14\,m^2\sigma^4}{C_4\,m^2 K_{\max}^4}
\ =:\ c_1 \in (0,1),
\]
hence $\mathbb P\!\left(|S_i|\ge (\sigma/2)\sqrt m\right)\ge c_1$.

For a high-probability lower bound, set $V_i^2:=\mathrm{Var}(S_i\mid x_i)=\sum_{j=1}^m (u^v_j)^2\,\mathrm{Var}(Z_j\mid x_i)\ge m\sigma^2$.
By Berry--Esseen (bounded third moments since $|Z_j|\le 2K_{\max}$),
\[
\sup_{x\in\mathbb R}\left|\mathbb P\!\left(\frac{S_i}{V_i}\le x \,\middle|\, x_i\right)-\Phi(x)\right|
\ \le\ \frac{C_{\mathrm{BE}}\,\sum_{j=1}^m \mathbb E\!\left[|u^v_j Z_j|^3\mid x_i\right]}{V_i^3}
\ \le\ \frac{C'}{\sqrt m},
\]
for a constant $C'$ depending only on $K_{\max}$ and $\sigma$.
Let $\tau_m:=m^{-\varepsilon/2}$. Since $V_i\ge \sigma\sqrt m$,
\[
\mathbb P\!\left(|S_i|\le \tau_m\sqrt m \,\middle|\, x_i\right)
\ \le\ 2\Phi\!\left(\frac{\tau_m\sqrt m}{V_i}\right)-1+\frac{C'}{\sqrt m}
\ \le\ \frac{2}{\sqrt{2\pi}}\frac{\tau_m\sqrt m}{V_i}+\frac{C'}{\sqrt m}
\ \le\ \frac{2}{\sqrt{2\pi}\,\sigma}\,m^{-\varepsilon/2}+\frac{C'}{\sqrt m}.
\]
Taking expectation over $x_i$ yields
\[
\mathbb P\!\left(|S_i|\le m^{\tfrac12-\tfrac{\varepsilon}{2}}\right)\ \le\ C''\,m^{-\varepsilon/2}+C'/\sqrt m\ \xrightarrow[m\to\infty]{}\ 0,
\]
for a constant $C''>0$. Consequently,
\[
|S_i|\ \ge\ m^{\tfrac12-\tfrac{\varepsilon}{2}}\quad\text{with probability }1-o(1),
\]
and in particular $|S_i|=\Omega_p(\sqrt m)$.

\medskip
\noindent\emph{Combining with the bound for $C$:}
Since $|C|=O_p\!\big(m^{\tfrac12-\varepsilon}\big)$ and $|S_i|\ge m^{\tfrac12-\tfrac{\varepsilon}{2}}$ with probability $1-o(1)$, it follows that
\[
\frac{|C|}{|S_i|}
\ =\ O_p\!\Big(\frac{m^{\tfrac12-\varepsilon}}{m^{\tfrac12-\tfrac{\varepsilon}{2}}}\Big)
\ =\ O_p\!\big(m^{-\tfrac{\varepsilon}{2}}\big)
\ \xrightarrow{p}\ 0.
\]
\end{proof}

\section{Implementation Details and Additional Results}
\label{sec:exp_detail}
\subsection{Implementation Details}
\emph{Architectures and hyperparameters.}
Network backbones and all optimization hyperparameters for our method and baselines match those in \cite{kim2021fine} to ensure fair comparison. Specifically, we use ResNet-34 models for CIFAR-10 and CIFAR-100.  We use the same set of hyper-parameters for CIFAR-10 and CIFAR-100. During training, we set a batch size of 128. We use SGD with a weight decay $5\times 10^{-4}$ and a momentum of $0.9$. The learning rate is initialized as $0.02$ and decrease by a factor of $10$ at epochs 40, 80 and 100. For Clothing 1M, we compare with SCE, ELR \citep{liu2020early}, DivideMix \citep{li2020dividemix}, CORES \citep{cheng2020learning}, FINE \citep{kim2021fine}, BHN \citep{yu2023delving}, RENT \citep{bae2024dirichlet}. We use ResNet-50 with weights pre-trained on ImageNet. The batch size is set to $64$. We use SGD with initial learning rate $0.01$, momentum $0.9$, and weight decay $5\times 10^e{-4}$. We train the neural network for $10$ epochs. 

\subsection{Additional Results}
\label{sec:appdxCIFARN}
Comparison of sample selection/reweighting methods on CIFAR-10n/100n are provided in Table~\ref{tab:cifar_n}.
For meta-learning, Table~\ref{tab:meta} reports MW-Net baselines under last-epoch and best-epoch selection. As shown in the table, MW-Net exhibits severe overfitting under label noise.
\providecommand{\acc}[2]{#1}

\begin{table*}[t]
  \centering
  \caption{Test Accuracy on CIFAR-10N/100N.}
  \renewcommand{\arraystretch}{1.15}
  \setlength{\tabcolsep}{6pt}
  \begin{tabular}{l ccccc c}
    \toprule
    \multirow{2}{*}{Method} & \multicolumn{5}{c}{CIFAR-10N} & CIFAR-100N \\
    \cmidrule(lr){2-6}\cmidrule(lr){7-7}
     & Aggre & Ran1 & Ran2 & Ran3 & Worst & Noisy \\
    \midrule
    Forward      & \acc{88.2}{?} & \acc{86.9}{?} & \acc{86.1}{?} & \acc{87.0}{?} & \acc{79.8}{?} & \acc{57.0}{?} \\
    JoCoR      & \acc{91.4}{?} & \acc{90.3}{?} & \acc{90.2}{?} & \acc{90.1}{?} & \acc{83.3}{?} & \acc{60.0}{?} \\
    CORES      & \acc{91.2}{?} & \acc{89.7}{?} & \acc{89.9}{?} & \acc{89.8}{?} & \acc{83.6}{?} & \acc{61.2}{?} \\
    CO-Teaching      & \acc{91.2}{?} & \acc{90.3}{?} & \acc{90.3}{?} & \acc{90.2}{?} & \acc{ 83.8}{?} & \acc{60.4}{?} \\
    CO-Teaching++      & \acc{90.6}{?} & \acc{89.7}{?} & \acc{ 89.5}{?} & \acc{89.5}{?} & \acc{83.3}{?} & \acc{57.9}{?} \\
    FINE      & \acc{91.0}{?} & \acc{90.1}{?} & \acc{89.9}{?} & \acc{90.1}{?} & \acc{81.1}{?} & \acc{58.4}{?} \\
    BHN & \acc{86.1}{?.?} & \acc{84.9}{?.?} & \acc{84.9}{?.?} & \acc{84.8}{?.?} & \acc{78.4}{?.?} & \acc{56.9}{?.?} \\
    RNT   & \acc{80.8}{?.?} & \acc{79.1}{?.?} & \acc{78.9}{?.?} & \acc{79.6}{?.?} & \acc{69.7}{?.?} & \acc{53.1}{?.?} \\
    \textbf{Ours}      & \acc{\textbf{92.3}}{?.?} & \acc{\textbf{91.6}}{?.?} & \acc{\textbf{91.3}}{?.?} & \acc{\textbf{91.2}}{?.?} & \acc{\textbf{85.6}}{?.?} & \acc{\textbf{64.7}}{?.?} \\
    \bottomrule
  \end{tabular}
  \label{tab:cifar_n}
\end{table*}

\begin{table}[t]
\centering
\setlength{\tabcolsep}{6pt}
\caption{Results on CIFAR-10/100 with symmetric and asymmetric label noise.}

\begin{tabular}{lcccccc}
\toprule
\textbf{Dataset}
  & \multicolumn{3}{c}{\textbf{CIFAR-10}}
  & \multicolumn{3}{c}{\textbf{CIFAR-100}} \\
\textbf{Noisy Type}
  & \multicolumn{2}{c}{Sym} & \multicolumn{1}{c}{Asym}
  & \multicolumn{2}{c}{Sym} & Asym \\
\textbf{Noise Ratio} & 20 & 50 & 40 & 20 & 50 & 40 \\
\midrule
MW-Net (last)               & \(85.8 \)              &  \(74.5 \)               & \( 77.8 \)              & \(63.0 \)              & \(46.5 \)              & \(45.3 \)              \\
MW-Net (best)               & \(91.2 \)              &  \(85.6 \)               & \(88.6 \)              & \(67.5 \)              & \(58.2 \)              & \(53.4 \)              \\
\textbf{Ours (FBR)}       & \(\bm{92.3}\)     & \({\bm{87.0}}\)     & \(\bm{90.6}\)     &  \(73.4\)      & \(\bm{65.4}\)      & \(\bm{73.2}\)     \\
\textbf{Ours (NTK)}       & \(91.4 \)     & \(86.4 \)     & \(89.7\)     &  \(\bm{73.6 }\)      & \(\bm{65.4 }\)      & \(73.1 \)     \\
\bottomrule
\end{tabular}
\label{tab:meta}
\end{table}
\end{document}